%% file: main.tex
\documentclass[11pt]{article}

\input{def}

\begin{document}

\title{Learning Integral Representations of Gaussian Processes}

\author{Zilong Tan \thanks{Work done while the author was at Duke University.}\\
Machine Learning Department\\
Carnegie Mellon University\\
Email: {\tt zilongt@cs.cmu.edu}
\and
Sayan Mukherjee\\
Departments of Statistical Science\\
Computer Science, Mathematics, \\
Biostatistics \& Bioinformatics \\
Duke University \\
Email: {\tt sayan@stat.duke.edu}}

\maketitle

\begin{abstract}
We propose a representation of Gaussian processes (GPs) based on powers of the integral operator defined by a kernel function, we call these stochastic processes
integral Gaussian processes (IGPs). Sample paths from IGPs are functions contained within the reproducing kernel Hilbert space (RKHS) defined by the kernel function, in contrast sample paths from the standard GP are not functions within the RKHS. We develop computationally efficient non-parametric regression models based on IGPs. The main innovation in our regression algorithm is the construction of a low dimensional subspace that captures the information most relevant to explaining variation in the response. We use ideas from supervised dimension reduction to compute this subspace. The result of using the construction we propose involves significant improvements in the computational complexity of estimating kernel hyper-parameters as well as reducing the prediction variance.
\end{abstract}

\section{Introduction}

Gaussian processes (GPs) \citep{Doob44} have been used extensively for non-parametric regression and density estimation in the statistics and machine learning literature \citep{Stein99,Rasmussen06}. Models based on GPs enjoy the flexibility that nonlinear interpolation and spatial structure can be naturally modeled. The covariance kernel of the GP is the crucial component that specifies the class of functions the GP can realize. A key observation that motivates this work is that the sample paths (realizations) of the GP with covariance kernel $\kappa\left(\cdot,\cdot\right)$ is outside the reproducing kernel Hilbert space (RKHS) $\mathcal{H}_{\kappa}$ with the reproducing kernel $\kappa$ almost surely by the zero-one law of GPs \citep{Kallianpur70,Driscoll73,Lukic01,Wahba90}. Thus, the GP perspective is natural for point estimates such as the posterior mean but is problematic for posterior sample paths from the RKHS \citep{Liang07,Pillai07,Dunson10}. In this paper, we present a construction of GPs with sample paths in a given RKHS, and develop fast algorithms for non-parametric regression. The utility of our approach involves posterior based hyper-parameter inference as well as variable selection in Bayesian kernel models \citep{Neal97,Bishop03,Dunson10,Crawford18}.

Let $\left\{f\left(x\right): x \in \mathcal{X}\right\}$ be a zero-mean GP indexed by a separable metric space $\mathcal{X}$. The standard GP regression model is specified by
\begin{align}
\label{eq:lr}
y_i = f\left(x_i\right) + \epsilon_i
\end{align}
with independent noise $\epsilon_i \sim \mathcal{N}\left(0,\sigma^2\right)$. Denote the covariance kernel as  $\kappa$ and the the RKHS induced by $\kappa$ as
$\mathcal{H}_{\kappa}$. Approaches such as \cite{Tipping01,Sollich02,Chakraborty12} assume that the $f\left(\cdot\right)$ is in a finite-dimensional subspace \begin{align*}
    \left\{f\left(\cdot\right) = \sum_{i=1}^n a_i \kappa\left(\cdot,x_i\right) \;\Big|\; x_i \in \mathcal{X}, a_i \in \mathbb{R} \right\}
\end{align*}
of $\mathcal{H}_\kappa$ dependent on the training data $\left\{x_i\right\}_{i=1}^n$. Generalizing the GP regression model \eqref{eq:lr} to the infinite-dimensional $\mathcal{H}_\kappa$, by directly placing the GP prior on $\mathcal{H}_\kappa$ is problematic since a sample path (realization) of the GP lies almost surely outside the RKHS $\mathcal{H}_{\kappa}$, i.e., $P\left(f\left(\cdot\right) \in \mathcal{H}_\kappa\right) = 0$ if $\mathcal{H}_\kappa$ is infinite dimensional (see e.g., Theorem 5.1 of \citealp{Kallianpur70}; \citealp{Driscoll73,Wahba90}).

One solution to this problem was given in \citep{Pillai07,Liang07} by assuming that $f\left(x\right)$ is specified by the following  integral operator
\begin{align}
\label{eq:int-rep}
    f\left(x\right) = \int_{\mathcal{X}}\kappa\left(x,z\right) \nu\left(z\right) d \pi\left(z\right)
\end{align}
where $\nu\left(\cdot\right)$ is the sample path form a GP and $\pi$ is a finite measure. Their approach can be viewed as the embedding of signed measures (see e.g., Chapter 4 of \citealp{Berlinet03}), and the resulting $f\left(\cdot\right)$ can be thought of as a GP
\citep{Ito54,Gelfand67}. A limitation of the GP specified by equation
\eqref{eq:int-rep} is that only a subspace of the functions in $\mathcal{H}_\kappa$ can be realized.
Indeed, the relation between the function space spanned by \eqref{eq:int-rep} and $\mathcal{H}_\kappa$ has been examined by considering integral operators based on powers of the covariance kernel \citep{Steinwart12}. We combine the integral representation in equation \eqref{eq:int-rep} with the idea of using powers of integral operators to propose a new GP model for which we can precisely characterize both the RKHS as well as sample paths of the functions drawn from the GP.

There are three main contributions of this paper with respect to both specifying GP models as well as efficient and accurate non-parametric regression based on GPs:
\begin{enumerate}
\item We first introduce a GP representation which expands the function space spanned by \eqref{eq:int-rep} while still having the sample paths confined to an RKHS specified by the kernel function. This is accomplished by using an integral operator defined by powers of kernels \citep{Steinwart12,Kanagawa18}. We demonstrate that the space of functions realized by this representation is strictly larger than those given by the specification in \eqref{eq:int-rep}. 

\item We then present a finite-sample variant of the GP that has nice properties in that the class of functions realized
include solutions to Tikhonov regularization problems \citep{Cucker02,Hofmann08} as well as several Bayesian kernel models \citep{Tipping01,Sollich02,Chakraborty12}. The key insight in formulating this variant is a dual construction of GPs
as the class of functions
\begin{align*}
    f\left(x\right) = \sum_{i=1}^n \beta_i \kappa\left(x,x_i\right), \quad \left(\beta_1,\beta_2,\cdots,\beta_n\right) \sim \mathcal{N}\left(\bm{0},\bm{\Sigma}_\beta\right)
\end{align*}
in contrast to the more classical construction of a GP based on eigenfunctions (elaborated in \cref{sec:GP-func})
\begin{align*}
    f\left(x\right) = \sum_{j=1}^l \alpha_j e_j\left(x\right), \quad \alpha_j \stackrel{i.i.d.}{\sim} \mathcal{N}\left(0,\lambda_j\right)
\end{align*}
where $\{e_j\}_{j=1}^l$ and $\left\{\lambda_j\right\}_{j=1}^l$ are the eigenfunctions and the corresponding eigenvalues for the integral operator defined by the kernel $\kappa$.

 \item To speed up computation as well as to potentially reduce prediction variance, we
 compute a low rank approximation of the RKHS using ideas from  sufficient dimension reduction (SDR) \citep{Li91,Cook98,Fukumizu04}. We show that the SDR method is natural and generally yields a low-rank parameterization of the covariance kernel. Fast algorithms are developed for computing this approximation. The overall computational complexity for learning the proposed low rank GP is $O\left(n^2 m\right)$ per iteration which improves the complexity $O\left(n^3\right)$ of the standard GP. Further speedups can be achieved by combining the algorithm with the sparse GP methods \citep{Smola01,Seeger03,Snelson06}.

\item Our approach is grounded in the theory of RKHS and sufficient dimension reduction. An exact characterization of the function space realized by the integral representation is given, suggesting that our approach is particularly useful for kernels with rapidly decaying eigenvalues. This holds true for several widely used kernels \citep{Zhu98,Belkin18}. Second, we generalize the classic SDR likelihood \citep{Cook09} on $\mathbb{R}^d$ to the RKHS setting. We justify that the RKHS can be well approximated by a low-rank SDR subspace if the eigenvalue of the reproducing kernel decays at a faster rate than $O\left(1/n\right)$.
\end{enumerate}

\subsection{Outline of the Paper}
The remainder of this paper is organized as follows: in \cref{sec:IGP}, we introduce the the GP representation using powers of the integral operators. In addition, we discuss computational issues and posterior inference based on the SDR approximation to the RKHS. In \cref{sec:theory}, we provide the theoretical justification for the results stated in \cref{sec:IGP} as well as the approximation to the RKHS using SDR subspaces. In \cref{sec:alg}, we develop fast algorithms for estimating the parameters of our approach. In \cref{sec:exp}, we illustrate the difference between our approach and the standard GP regression model, and report competitive experimental results on a diverse collection of real-life datasets. In \cref{sec:concl}, we conclude this paper.

\subsection{Notation}
We use bold lowercase letters for vectors, bold capital letters for matrices, calligraphic letters for sets. The $i$-th row and $j$-th column of a matrix $\bm{A}$ are denoted by $\bm{A}_{i:}$ and $\bm{A}_j$, respectively. In addition, we write $\bm{A}_\bot$ for the orthogonal complement to the column space of $\bm{A}$, $\tr\left(\cdot\right)$ for the trace, and $\det\left(\cdot\right)$ for the determinant. $\bm{I}_n$ and $\bm{1}_n$ denote respectively the $n$-by-$n$ identity matrix and a (column) $n$-vector of all ones. Random variables $X$ and $Y$ represent the covariates vector and the response, whose realizations are the matrix $\bm{X}$ and $\bm{y}$ with the $i$-th row corresponding to the $i$-th observation. Similarly, $\kappa\left(\bm{X},\bm{Z}\right)$ denotes the matrix whose $\left(i,j\right)$-th element is specified by $\kappa\left(\bm{X}_{i:},\bm{Z}_{j:}\right)$. We also denote by $\bm{K}_{\kappa}$ the kernel matrix of a kernel function $\kappa$, whose $\left(i,j\right)$-th entry is specified by $\kappa\left(x_i,x_j\right)$.

\section{Integral Representations of Gaussian Processes}
\label{sec:IGP}

Let $\left(\mathcal{X},\mathcal{B},\mu\right)$ be a measure space for some finite measure $\mu$. Throughout this paper, we assume that $\kappa$ is a continuous positive definite kernel measurable on $\mathcal{X}\times\mathcal{X}$. In addition, $\kappa$ has an infinitely number of nonzero eigenvalues, and is of trace class, i.e., the sum of eigenvalues is finite, a property that characterizes covariance kernels (see e.g., Section 2.3 of \citealp{Horvath12}).

We consider the GP obtained by applying an integral operator $\mathscr{K}^p$ to a GP $\nu: \mathcal{X} \times \Omega \mapsto \mathbb{R}$ on a probability space $\left(\Omega,\mathcal{F},P\right)$ with covariance kernel $\kappa_{\nu}$. The integral operator $\mathscr{K}^p: \mathcal{L}^2\left(\mathcal{X},\mathcal{B},\mu\right) \mapsto  \mathcal{L}^2\left(\mathcal{X},\mathcal{B},\mu\right)$ is based on the powers of kernels \citep{Steinwart12,Kanagawa18}, and is given by
\begin{align*}
\mathscr{K}^p \nu \coloneqq \int_{\mathcal{X}} \kappa_p\left(\cdot,z\right) \nu\left(z\right) d \mu\left(z\right)\quad \text{with} \quad \kappa_p\left(x,z\right) \coloneqq \sum_{i=1}^\infty \lambda_i^p e_i\left(x\right) e_i\left(z\right), \quad \frac{1}{2} \leq p \leq 1,
\end{align*}
where $\left\{e_i\right\}$ denotes the complete set of orthonormal eigenfunctions of $\kappa$ corresponding to the non-increasing sequence of eigenvalues $\left\{\lambda_i\right\}$. This integral representation is an instance of generalized stochastic processes introduced in \citep{Ito54,Gelfand67}, where $\kappa_p\left(x,z\right)$ can be viewed as a weighting function. The resulting GP $\mathscr{K}^p \nu$ is written
\begin{align}
\label{eq:IGP-def}
\mathcal{M} \coloneqq \left\{f\left(x\right) \coloneqq \sum_{i=1}^\infty \lambda_i^p \varphi_i e_i\left(x\right) \;\Big|\; \varphi_i \coloneqq \int_{\mathcal{X}} \nu\left(z\right) e_i\left(z\right) d \mu\left(z\right), x\in\mathcal{X} \right\}.
\end{align}
In particular, $\mathcal{M}$ is equivalent to \eqref{eq:int-rep} for the case that $p=1$. Setting $p<1$, the representation \eqref{eq:IGP-def} becomes a superset of \eqref{eq:int-rep} but can still have all sample paths confined within $\mathcal{H}_{\kappa}$ as we will show in \cref{sec:func-spc}. In the remainder of this paper, the GP \eqref{eq:IGP-def} will be termed the Integral GP (IGP), and $\gamma\left(f; \kappa_p,\kappa_{\nu}\right)$ denotes the corresponding prior. 

Let us consider some basic properties of the IGP. For an IGP $f\in\mathcal{M}$, the curve $f\left(x\right)$, $x \in \mathcal{X}$, is a random element of $\mathcal{L}^2\left(\mathcal{X},\mathcal{B},\mu\right)$, and is square integrable since
\begin{align*}
\mathbb{E}\left\|f\right\|^2 
= \mathbb{E} \sum_{i=1}^\infty \left(\lambda_i^p\int_{\mathcal{X}} \nu\left(z\right) e_i\left(z\right) d \mu\left(z\right)\right)^2
\leq \lambda_1^{2 p} \mathbb{E}\left\|\nu\right\|_{\mathcal{L}^2}^2
< \infty
\end{align*}
holds by Parseval's identity. Its mean and covariance functions are written
\begin{gather}
\mathbb{E}f = \sum_{i=1}^\infty \lambda_i^p \mathbb{E}\left(\varphi_i\right) e_i \label{eq:IGP-mean}\\
\kappa_{\mathcal{M}}\left(x,z\right) 
= \int_{\mathcal{X}}\int_{\mathcal{X}} 
\kappa_{\nu}\left(s,t\right) \kappa_p\left(x,s\right) \kappa_p\left(z,t\right) d \mu\left(s\right) d \mu\left(t\right). \label{eq:IGP-cov}
\end{gather}
Let $\bm{K}$ and $\bm{K}_{\kappa_{\mathcal{M}}}$ be respectively the kernel matrix $K_{ij} = \kappa\left(x_i,x_j\right)$ and the covariance matrix $\left[\bm{K}_{\kappa_{\mathcal{M}}}\right]_{ij} \coloneqq \kappa_{\mathcal{M}}\left(x_i,x_j\right)$ on a finite sample $\left\{x_i\right\}_{i=1}^n$. By analogy, \cref{prop:sample-cov} gives the covariance matrix $\bm{K}_{\kappa_{\mathcal{M}}}$. The proof is straightforward by observing that $\lambda_i = \eta_i/n$ and $e_i\left(x_j\right) = \sqrt{n}b_{ij}$ for eigenvalue-eigenvector pairs $\left(\eta_i,\bm{b}_{i}\right)$ of $\bm{K}$.
\begin{prop}
\label{prop:sample-cov}
Suppose that $\mu = \frac{1}{n}\sum_{i=1}^n \delta_{x_i}$ where $\delta_{x_i}$ is the Dirac measure, the covariance matrix of \eqref{eq:IGP-cov} is given by $\bm{K}_{\kappa_\mathcal{M}} = n^{-2p} \, \bm{K}^p \bm{K}_\nu \bm{K}^p$.
\end{prop}

Another perspective is that the covariance function $\kappa_{\mathcal{M}}$ is non-stationary which enables the IGP to model spatial data. For example, consider the setting $\kappa_{\nu} = \kappa$ and $p = 1/2$ in \eqref{eq:IGP-def}, the covariance kernel of the IGP is simplified to
\begin{align}
\label{eq:cov-kern}
\kappa_{\mathcal{M}}\left(x,z\right) = \int_{\mathcal{X}} \kappa\left(x,s\right) \kappa\left(z,s\right) d \mu\left(s\right).
\end{align}
This gives us a non-stationary covariance kernel since it is no longer a function of the distance alone. Note that we do not require $\kappa$ to be non-stationary as well. The covariance kernel \eqref{eq:cov-kern} can be viewed as the correlation between two points $x$ and $z$ taking into account neighborhood locations $s$. In this example, \eqref{eq:cov-kern} is the spatial correlation kernel proposed in \citep{Higdon98}.

\subsection{Characterizing the Function Space}
\label{sec:func-spc}
Recall that our goal is to construct a GP whose sample paths are confined in a given RKHS, and can realize the largest possible subspace of the RKHS. In this section, we outline how the IGP meets this goal, and defer the proof until \cref{sec:proofs}.

First, the IGP representation \eqref{eq:IGP-def} can realize an expanding subspace of $\mathcal{H}_\kappa$ as $p$ decreases, and the entire $\mathcal{H}_\kappa$ for $p = 1/2$ and some $\nu \in \mathcal{L}^2\left(\mathcal{X}\times\Omega,\mathcal{F},P\right)$. \cref{prop:func-inf} characterizes the function space of the IGP. The proof of the proposition is straightforward using Mercer's representation of the RKHS discussed in \cref{sec:rkhs}.

\begin{prop}
\label{prop:func-inf}
Suppose that $\left\{\nu\left(x\right): x \in \mathcal{X}\right\}$ is a GP on $\mathcal{L}^2\left(\Omega,\mathcal{F},P\right)$. Then, the space of functions realizable by the IGP \eqref{eq:IGP-def} is given by $\left\{\sum_{i=1}^\infty w_i e_i : \sum_{i=1}^\infty w_i^2/\lambda_i^{2 p} < \infty\right\}$.
\end{prop}

\cite{Pillai07} and \cite{Liang07} proposed the following class of functions
$$f\left(x\right) = \int_{\mathcal{X}} \kappa\left(x,z\right) \nu\left(z\right) d \mu \left(z\right) = \sum_{i=1}^\infty \lambda_i e_i\left(x\right) \int_{\mathcal{X}} \nu\left(z\right) e_i\left(z\right) d \mu \left(z\right)$$ which corresponds to the IGP with $p=1$. From \cref{prop:func-inf}, the IGP can realize an expanded space of functions setting $p<1$. This also holds when the Hilbert space of $\nu\left(\cdot\right)$ is an RKHS subspace of the $\mathcal{L}^2$ (see \cref{sec:GP-func}).

While a smaller $p$ allows the IGP to realize more functions, the sample path can potentially go beyond $\mathcal{H}_\kappa$. \cref{thm:IGP-sample} provides a sufficient condition for the IGP sample paths to be confined within $\mathcal{H}_\kappa$, requiring the eigenvalue of $\kappa$ to decay rapidly. We demonstrate this condition is satisfied for several popular classes of kernels including radial kernels.

\begin{theorem}
\label{thm:IGP-sample-inf}
Let $f \in \mathcal{L}^2\left(\mathcal{X}\times\Omega,\mathcal{F},P\right)$ be an IGP. Then, it holds that $P\left(f\left(\cdot\right) \in \mathcal{H}_{\kappa}\right) = 1$ if the kernel $\kappa_{2p - 1}$ is of trace class.
\end{theorem}

From \cref{thm:IGP-sample-inf}, sample paths of the IGP with $p=1$ are confined within $\mathcal{H}_\kappa$ almost surely by our assumption that $\kappa$ is of trace class. This corresponds to the case where the function space of IGP is minimized from \cref{prop:func-inf}. For several popular classes of kernel operators, it is possible for the IGP to realize a larger space of functions while still having its sample paths confined in $\mathcal{H}_\kappa$. Below, we give two examples. 

\paragraph{Radial Kernels} For any $p>1/2$, the condition of \cref{thm:IGP-sample-inf} is satisfied for radial kernels of which the eigenvalue decays nearly exponentially \citep{Belkin18}. For example, consider $\kappa\left(x,z\right) = \exp\left(-\frac{\left\|x - z\right\|^2}{2 \ell^2}\right)$ and $\mu\left(x\right)$ is Gaussian, it has been shown
\begin{gather*}
\lambda_j \propto b^j, \quad b < 1\\
e_j\left(x\right) \propto \exp\left(-\left(c-a\right) x^2\right) h_j\left(x\sqrt{2 c}\right),
\end{gather*}
where $a,b,c$ are functions of $\ell$, and $h_j$ is the $j$-th order Hermite polynomial \citep{Zhu98}.

\paragraph{Brownian Bridge}
Suppose that $\mathcal{X} = \left[0,1\right]$ and $\mu$ is the Lebesgue measure, the Brownian bridge kernel reads $\kappa\left(x,z\right) = \min\left(x,z\right) - x z$ \citep{Rogers00}. The corresponding eigenvalues and eigenfunctions are given by
\begin{align*}
\lambda_j = \frac{1}{\pi^2 j^2}, \quad
e_j\left(x\right) = \sqrt{2}\sin\left(j \pi x\right).
\end{align*}
Thus, the condition of \cref{thm:IGP-sample-inf} holds for $p > 3/4$.

\subsection{A Finite-Sample Variant}
\label{sec:digp}

To obtain practical algorithms, we also present a finite-sample variant of the IGP (SIGP) on a sample $\left\{x_i\right\}_{i=1}^n$. The function space of the SIGP will be specified by the Hilbert closure $\mathcal{H}_{\kappa,n}$ of the linear subspace
\begin{align*}
    \left\{\sum_{i=1}^n a_i \kappa\left(\cdot,x_i\right) \;\Big|\; \left\{a_i\right\}_{i=1}^n \subset \mathbb{R}, \left\{x_i\right\}_{i=1}^n \subset \mathcal{X}\right\}.
\end{align*}
Suppose that $x_i$'s in the sample $\left\{x_i\right\}_{i=1}^n$ are i.i.d.\  and $\mu$ is a probability measure without loss of generality. The SIGP is defined as
\begin{align}
\label{eq:SIGP-def}
f_n\left(\cdot\right) \coloneqq \frac{1}{n} \sum_{i=1}^n \nu\left(x_i\right) \kappa_p\left(\cdot,x_i\right),
\end{align}
which converges to the IGP at a rate $O_P\left(n^{-1/2}\right)$ by the Central Limit Theorem for Hilbert spaces \citep{Ledoux91,Berlinet03}. Compared to the standard function-space view of GPs (see \citealp{Williams97} and \cref{sec:GP-func}), the SIGP representation provides a dual construction of GPs where the Gaussian coefficients are assigned to the representer $\kappa_p\left(\cdot,x_i\right)$ of each data point $x_i$ as opposed to each eigenfunction $e_i\left(\cdot\right)$ in the function-space view of GPs.

The SIGP can be useful for two reasons. First, the dual construction of GPs offered by the SIGP affords natural approximations and thereby fast computation, as will be discussed in \cref{sec:approx}. Second, the SIGP function space $\mathcal{H}_{\kappa,n}$ contains ``interesting" functions including the solution to Tikhonov regularization problems by the representer theorem:
\begin{align}
\label{eq:reg-obj}
\min_{f \in \mathcal{H}_{\kappa}} L\left(f,\mathcal{D}\right) + \rho\left(\left\|f\right\|_{\mathcal{H}_{\kappa}}\right)
\end{align}
with a convex loss function $L$ and a monotonically increasing function $\rho$, which enjoys wide application in machine learning \citep{Cucker02,Hofmann08}. This also includes Bayesian kernel methods based on placing a prior on $\mathcal{H}_{\kappa,n}$ \citep{Tipping01,Sollich02,Chakraborty12}. These approaches implicitly assume that the prior distribution is supported on the finite-dimensional subspace $\mathcal{H}_{\kappa,n}$ (see also \citealp{Dunson10} for details).

We point out that the finite-sample variant can be reduced to the setting $p=1$ since $\kappa_p$ and $\kappa$ have the same eigenfunctions. This fact leads to simplified expressions and fast computation. Specifically, let $\bm{K} = \bm{U}\bm{D}\bm{U}^\top$ be the eigenvalue decomposition, we can rewrite $$\kappa_{\mathcal{M}}\left(x_i,x_j\right) = n^{-2p}\kappa\left(x_i,\bm{X}\right)\bm{U}\bm{D}^{p-1} \bm{U}^\top \bm{K}_{\nu} \bm{U} \bm{D}^{p-1}\bm{U}^\top\kappa\left(\bm{X},x_j\right).$$ This is always possible because $\bm{K}$ is positive definite and hence $\bm{D}$ is invertible. For a centered $f_n\left(\cdot\right)$ on the sample, it can be expressed as a linear combination of the centered representers $\bar{\kappa}\left(\cdot,x_i\right) \coloneqq \kappa\left(\cdot,x_i\right) - \sum_{j=1}^n \kappa\left(\cdot,x_j\right)/n$. Denote by $\bm{\Gamma}_n \coloneqq \bm{I}_n - \bm{1}_n\bm{1}_n^\top/n$, we have $\bar{\kappa}\left(\cdot,\bm{X}\right) = \kappa\left(\cdot,\bm{X}\right) \bm{\Gamma}_n$. Using the parameterization $\bm{\Sigma} \coloneqq n^{-2p}\bm{U}\bm{D}^{p-1} \bm{U}^\top \bm{K}_{\nu} \bm{U} \bm{D}^{p-1}\bm{U}^\top$, the centered SIGP is simplified to 
\begin{align}
\label{eq:SIGP}
f_n\left(\cdot\right) \sim \mathcal{GP}\left(\bm{0}, \kappa\left(\cdot,\bm{X}\right)\bm{\Gamma}_n\bm{\Sigma}\bm{\Gamma}_n \kappa\left(\bm{X},\cdot\right)\right),
\end{align}
where the matrix $\bm{\Sigma}$ is the parameter to be estimated.

\subsection{Approximation with Sufficient Dimension Reduction Subspaces}
\label{sec:approx}

Recall that the SIGP is supported on $\mathcal{H}_{\kappa,n}$, we may intuitively adopt supervised dimension reduction of $\mathcal{H}_{\kappa,n}$ to speedup the computation as well as reduce the variance in the prediction. A natural choice for this task is the sufficient dimension reduction (SDR) which is a state-of-the-art approach to supervised dimension reduction \citep{Li91,Cook91,Cook98,Fukumizu04,Fukumizu09,Wu09,Cook09}. Below, we first briefly review the basics of the SDR.

\paragraph{Background on SDR} In a regression problem with response $Y$ and covariate $X$, the SDR aims to find a subspace $\mathcal{S}$ such that the projection of $X$ onto $\mathcal{S}$ captures the statistical dependency of $Y$ on $X$. The SDR is typically stated as the conditional independence 
\begin{align}
\label{eq:sdr-dr}
Y \ci X \mid \mathscr{P}_{\mathcal{S}} X \quad \text{or} \quad P\left(Y\mid X\right) = P\left(Y \mid \mathscr{P}_{\mathcal{S}} X\right)
\end{align}
where $\mathscr{P}_{\mathcal{S}}$ denotes the orthogonal projection onto $\mathcal{S}$. The SDR can be viewed as the supervised version of principle component analysis that takes into account the information of the response.

Suppose that $X \in \mathbb{R}^d$ and $\bm{S} \in \mathbb{R}^{d\times m}$, and let $\bm{S}_1^\top X,\bm{S}_2^\top X,\cdots,\bm{S}_m^\top X$ be the projection onto $\mathcal{S}$. The SDR typically assumes that $X$ satisfies \cref{cond:ldc}, a property characterizing elliptically symmetric distributions such as the normal distribution \citep{Eaton83}.
\begin{cond}
\label{cond:ldc}
For any $\bm{a}\in \mathbb{R}^d$, there exists $c_0 \in \mathbb{R}$ and $\bm{c}\in \mathbb{R}^m$ such that
$\mathbb{E}\left(\bm{a}^\top X \mid \bm{S}^\top X\right) = c_0 + \bm{c}^\top \bm{S}^\top X$.
\end{cond}
This condition leads to an important result stated in \cref{thm:sir} which gives an algorithm for computing $\bm{S}$ for an {\em arbitrary unknown} link function $f$ in \eqref{eq:lr}.
\begin{theorem}[\citealp{Li91}]
\label{thm:sir}
Under assumption \eqref{eq:sdr-dr} and \cref{cond:ldc}, it holds that $\mathbb{E}\left(X\mid Y\right) - \mathbb{E} X \in \text{span}\left(\text{Var}\left(X\right)\bm{S}\right)$.
\end{theorem}
\cref{thm:sir} implies that $\bm{S}$ is given by the leading eigenvector of the generalized eigenvalue decomposition $\text{Var}\left(\mathbb{E}\left(Y\mid X\right)\right) \bm{S} = \text{Var}\left(X\right) \bm{S} \bm{\Lambda}$, where $\bm{\Lambda}$ is the diagonal eigenvalue matrix. We will derive an RKHS variant of the theorem in \cref{sec:sdr-rkhs}, and show that the SIGP distribution satisfies the SDR assumptions. The intuition is that the analytical expression of the SIGP \eqref{eq:SIGP-def} is a linear combination of Gaussian random functions $\kappa_p \left(\cdot,x_i\right)$ which can be considered as the covariates in $\mathcal{H}_{\kappa,n}$.

\paragraph{SDR Approximation to SIGP} We highlight two key results on the SDR approximation. First, the SIGP can be well-approximated with a low-rank SDR subspace when the eigenvalue of $\kappa$ decays at a faster rate than $O\left(1/n\right)$. This is consistent with the observation that many widely-used kernel operators have nearly exponential eigenvalue decay \citep{Belkin18}. Second, we derive a log-likelihood 
\begin{align}
g\left(\bm{W}\right) = - \frac{n}{2} \log \frac{\det \left( \bm{W}^\top \bm{M} \bm{W} \right)}
{\det \left(\bm{W}^\top \bm{N} \bm{W}\right)}
\label{eq:lik-rkhs-inf}
\end{align}
for the SDR subspace of $\mathcal{H}_{\kappa,n}$ spanned by the basis
\begin{align}
\label{eq:basis}
\sum_{i=1}^n W_{i1} \kappa\left(\cdot,x_i\right), \sum_{i=1}^n W_{i2}\kappa\left(\cdot,x_i\right), \cdots, \sum_{i=1}^n W_{im}\kappa\left(\cdot,x_i\right),
\end{align}
where $\bm{W}$ specifies the coefficients of the basis functions of the SDR subspace, and $\bm{M}$ and $\bm{N}$ are computed from the kernel matrix $\bm{K}$. The log-likelihood \eqref{eq:lik-rkhs-inf}
extends the classic likelihood for SDR subspaces of $\mathbb{R}^d$ \citep{Cook09} to the RKHS setting. These results will be proved in \cref{sec:theory}.

The approximation to the centered SIGP \eqref{eq:SIGP} using basis \eqref{eq:basis} is given by a random function $f_n \approx n^{-1} \sum_{j=1}^m \beta_j \sum_{i=1}^n \bm{W}_{ij} \kappa\left(\cdot,x_i\right)$ with $\bm{\beta} \sim \mathcal{N}\left(\bm{0},\bm{\Sigma}_\beta\right)$, and the reproducing property states
\begin{align*}
f_n\left(z\right) \approx \left<\bar{\kappa}\left(\cdot,z\right),\sum_{j=1}^m \beta_j \sum_{i=1}^n \bm{W}_{ij} \kappa\left(\cdot,x_i\right)\right>_{\mathcal{H}_\kappa}
= \left(\kappa\left(z,\bm{X}\right) - \frac{1}{n}\bm{1}\bm{1}_n^\top\bm{K}\right)\bm{W}\bm{\beta}.
\end{align*}
With this approximation, we will instead estimate the parameters of the finite $\bm{\beta}$ in the approximated SIGP. It is worth pointing out that $\bm{\Pi}\left(z\right) \coloneqq \left(\kappa\left(z,\bm{X}\right) - \bm{1}\bm{1}_n^\top\bm{K}/n\right)\bm{W}$ is the evaluation of the projection of the centered SIGP at $z$. The approximate SIGP distribution is therefore
\begin{align}
\label{eq:apx-SIGP}
f_n\left(\cdot\right) \sim \mathcal{G P}\left(\bm{0},
\bm{\Pi}\left(\cdot\right)\bm{\Sigma}_\beta \bm{\Pi}\left(\cdot\right)^\top\right).
\end{align}
In the subsequent sections, we will assume that the approximate SIGP \eqref{eq:apx-SIGP} is used.

\subsection{The Predictive Distribution}
\label{sec:predict}

Let us first consider the predictive distribution of the vanilla SIGP \eqref{eq:SIGP-def}. The posterior induced by the model \eqref{eq:lr} with the SIGP \eqref{eq:SIGP-def} is given by
\begin{align}
\label{eq:posterior}
\gamma\left(f_n \mid \mathcal{D}\right) \propto \frac{\gamma\left(f_n; \kappa_p,\kappa_{\nu}\right)}{\left(2\pi\right)^{n/2} \sigma^n} \prod_{i=1}^n  \exp\left[-\frac{1}{2\sigma^2}\left(y_i - f_n\left(x_i\right)\right)^2\right],
\end{align}
where $\gamma\left(f_n; \kappa_p,\kappa_{\nu}\right)$ denotes the SIGP prior on a sample $\left\{x_i\right\}_{i=1}^n$. Thus, it remains to obtain the analytical expression of $\gamma\left(f_n; \kappa_p,\kappa_{\nu}\right)$. Similar to \cref{prop:sample-cov}, we have the covariance $\text{Cov}\left(f_n\left(x_i\right),f_n\left(x_j\right)\right) = n^{-2p} \left[\bm{K}^p\bm{K}_\nu\bm{K}^p\right]_{ij}$ where $\left[\bm{K}_{\nu}\right]_{i j} \coloneqq \kappa_{\nu}\left(x_i,x_j\right)$ is the covariance matrix of $\nu\left(\cdot\right)$. The SIGP prior is then expressed as the density of the multivariate Gaussian distribution $\gamma\left(f_n; \kappa_p,\kappa_{\nu}\right) = \mathcal{N}\left(\bm{f} \mid \bm{0}, n^{-2p} \bm{K}^p\bm{K}_\nu\bm{K}^p\right)$. Denote by $\bm{f} \coloneqq \left(f\left(x_1\right), f\left(x_2\right),\cdots,f\left(x_n\right)\right)^\top$ a column vector. The reproducing property yields $\bm{f} = \bm{K}\bm{\alpha}$ for some $\bm{\alpha} \in \mathbb{R}^n$. Substitute $\bm{f} = \bm{K}\bm{\alpha}$ into the SIGP prior density, we arrive at
\begin{align}
\label{eq:IGP-prior}
\gamma\left(\bm{\alpha}; \kappa_p,\kappa_{\nu}\right)
= \mathcal{N}\left(\bm{\alpha} \;\Big|\; \bm{0}, n^{-2p} \bm{K}^{p-1}\bm{K}_{\nu}\bm{K}^{p-1}\right).
\end{align}
Combined with \eqref{eq:posterior}, one can select the kernel $\bm{K}_\nu$ using the maximum marginal likelihood method as in the standard GP \citep{Rasmussen06}.

Now consider the the approximate SIGP \eqref{eq:apx-SIGP}. Denote by $\bm{X}_T$ and $\bm{y}_T$ the covariate matrix and response vector on the test set, respectively. The predictive distribution of $\bm{y}_T$ is a multivariate Gaussian, and we have the function covariance
\begin{align*}
\text{Cov}\left(f_n\left(\bm{X}_T\right),f_n\left(\bm{X}\right)\right) = \bm{\Pi}\left(\bm{X}_T\right) \bm{\Sigma}_\beta \bm{\Pi}\left(\bm{X}\right)^\top.
\end{align*}
Denote by $\bm{\Sigma}_{Z Y} \coloneqq \text{Cov}\left(f_n\left(\bm{Z}\right),f_n\left(\bm{Y}\right)\right)$ the function covariance matrix, the mean of the predictive SIGP distribution is given by
\begin{align*}
\mathbb{E} \left(\bm{y}_T \mid \mathcal{D}\right)
&= u\left(\bm{X}_T\right) + \bm{\Sigma}_{X_T X}\bm{V}^{-1}\left(\bm{y} - u\left(\bm{X}\right)\right)\\
&= u\left(\bm{X}_T\right) + \bm{\Pi}\left(\bm{X}_T\right)\widehat{\bm{\beta}},
\end{align*}
where $u\left(\bm{X}_T\right)$ is the mean function, and $\bm{V} \coloneqq \bm{\Sigma}_{XX} + \sigma^2\bm{I}$ represents the marginal variance of the SIGP.
 
Let $\bm{P}_{Z} = \bm{\Pi}\left(\bm{Z}\right)\bm{\Sigma}_\beta^{1/2}$ such that $\bm{\Sigma}_{Z Y} = \bm{P}_{Z}\bm{P}_{Y}^\top$. The variance of the predictive distribution can be written
\begin{align}
\label{eq:pred-var}
\begin{split}
\text{Var}\left(\bm{y}_T \mid \mathcal{D}\right) 
&= \bm{\Sigma}_{X_T X_T} - \bm{\Sigma}_{X_T X} \bm{V}^{-1} \bm{\Sigma}_{X X_T} + \sigma^2\bm{I}\\
&= \bm{P}_{X_T} \bm{P}_{X_T}^\top - \bm{P}_{X_T} \bm{P}_X^\top \left(\bm{P}_X \bm{P}_X^\top + \sigma^2 \bm{I}\right)^{-1} \bm{P}_X \bm{P}_{X_T}^\top + \sigma^2\bm{I}\\
&= \bm{P}_{X_T} \bm{P}_{X_T}^\top - \bm{P}_{X_T}\left[\bm{I} - \left(\sigma^{-2}\bm{P}_X^\top \bm{P}_X + \bm{I}\right)^{-1}\right]\bm{P}_{X_T}^\top + \sigma^2\bm{I}\\
&= \bm{P}_{X_T}  \left(\sigma^{-2}\bm{P}_X^\top\bm{P}_X + \bm{I}\right)^{-1} \bm{P}_{X_T}^\top + \sigma^2\bm{I}\\
&= \bm{\Pi}\left(\bm{X}_T\right)\bm{\Delta}\bm{\Pi}\left(\bm{X}_T\right)^\top + \sigma^2\bm{I}
\end{split}
\end{align}
with $\bm{\Delta} \coloneqq \left(\bm{\Sigma}_\beta^{-1} + \sigma^{-2}\bm{\Pi}\left(\bm{X}\right)^\top\bm{\Pi}\left(\bm{X}\right)\right)^{-1} \in \mathbb{R}^{m\times m}$. We only need the diagonal elements of $\text{Var}\left(\bm{y}_T\mid \mathcal{D}\right)$, among which the $i$-th diagonal entry is given by $\left\|\left(\bm{\Pi}\left(\bm{X}_T\right)\bm{\Delta}^{1/2}\right)_{i:} \right\|^2  + \sigma^2$. These diagonal entries can be computed efficiently for low-rank SDR approximation with a small $m$.

\section{Theoretical Underpinnings and Results}
\label{sec:theory}

In this section, we begin by providing the background on the RKHS as well as the function space of GPs. Then, we present the proofs for the main results outlined in \cref{sec:func-spc} and \cref{sec:approx}. Finally, we derive a likelihood for the SDR approximation in the SIGP which can be of independent interest.

\subsection{Reproducing Kernel Hilbert Spaces and Gaussian Processes}
\label{sec:rkhs}

We review some relevant concepts of the RKHS. More comprehensive exposition of the theory of RKHS can be found in e.g., \citep{Aronszajn50,Parzen70,Cucker02,Berlinet03} and the references therein. An RKHS $\mathcal{H}$ is a Hilbert space of functions (on $\mathcal{X}$) with bounded evaluation functionals $\delta_x\left(f\right) \coloneqq f\left(x\right)$ for all $f \in \mathcal{H}$. Below, we recall three important properties of the RKHS.

First, the RKHS has the reproducing property. From Riesz representation theorem, the RKHS is ``{\em reproducing}"  in the sense that for every $x \in \mathcal{X}$, there exists a unique {\em representer} function $\phi\left(x\right) \in \mathcal{H}$ satisfying
\begin{align*}
\left(\forall f \in \mathcal{H}\right) \quad f\left(x\right) = \left<f,\phi\left(x\right)\right>_{\mathcal{H}},
\end{align*}
where $\left<\cdot,\cdot\right>_{\mathcal{H}}$ denotes the inner product in $\mathcal{H}$. Each input data $x \in \mathcal{X}$ is mapped to the corresponding representer function $\phi\left(x\right)$ via the map $\phi: \mathcal{X} \mapsto \mathcal{H}$ which is commonly referred to as the {\em feature map}.

Another important property of the RKHS is the one-to-one correspondence between the RKHS and the kernel. Apply the reproducing property to $\phi\left(x\right)$, one has
\begin{align*}
    \left[\phi\left(x\right)\right]\left(x^\prime\right) = \left<\phi\left(x^\prime\right) ,\phi\left(x\right) \right>_{\mathcal{H}} =
    \left<\phi\left(x\right) ,\phi\left(x^\prime\right) \right>_{\mathcal{H}} =
    \left[\phi\left(x^\prime\right)\right] \left(x\right).
\end{align*}
Suppose that $\mathcal{X}$ is a separable metric space, the {\em reproducing kernel} $\kappa: \mathcal{X}\times \mathcal{X} \mapsto \mathbb{R}$ of $\mathcal{H}$ is a continuous symmetric function given by 
\begin{align*}
\kappa\left(x,x^\prime\right) \coloneqq
\left<\phi\left(x\right),\phi\left(x^\prime\right)\right>_{\mathcal{H}}
= \left<\kappa\left(\cdot,x\right),\kappa\left(\cdot,x^\prime\right)\right>_{\mathcal{H}}.
\end{align*}
Note that $\kappa\left(\cdot,\cdot\right)$ is also positive semidefinite since for any countable subset $\left\{x_i\right\}$ dense in $\mathcal{X}$ and $\left\{a_i\right\} \subset \mathbb{R}$, we have $
\sum_{i j} a_i a_j \left<\phi\left(x_i\right),\phi\left(x_j\right)\right>_{\mathcal{H}}
= \left<\sum_{i=1}^\infty a_i \phi\left(x_i\right),\sum_{i=1}^\infty a_i \phi\left(x_i\right)\right>_{\mathcal{H}} \geq 0$.
This shows that for each RKHS $\mathcal{H}$, there is a unique symmetric positive semidefinite reproducing kernel $\kappa\left(\cdot,\cdot\right)$. The converse is also true --- the Hilbert closure of
\begin{align}
\label{eq:finite-rkhs}
\left\{\sum_{i=1}^n a_i \kappa\left(\cdot,x_i\right) \;\Big|\; n \in \mathbb{N}, \left\{a_i\right\}_{i=1}^n \subset \mathbb{R},\left\{x_i\right\}_{i=1}^n \subset \mathcal{X}\right\}
\end{align}
under the norm $\left\|\sum_{i=1}^\infty a_i \kappa\left(\cdot,x_i\right)\right\| = \sqrt{\sum_{i,j} a_i a_j \kappa\left(x_i,x_j\right)}$ is the unique RKHS $\mathcal{H}_{\kappa,n}$ of $\kappa\left(\cdot,\cdot\right)$ \citep{Aronszajn50}.

Last but not least, the RKHS can be related to an $\mathcal{L}^2\left(\mathcal{X},\mathcal{B},\mu\right)$ space $\mathcal{G}$ of measurable real-valued functions on $\mathcal{X}$ equipped with the inner product 
$\left<f,g\right>_{\mathcal{L}^2} \coloneqq \int_{\mathcal{X}} f\left(x\right) g\left(x\right) d\mu\left(x\right)$ and the generated norm $\left\|f\right\|_{\mathcal{L}^2} \coloneqq \left(\int_{\mathcal{X}} f\left(x\right) g\left(x\right) d\mu\left(x\right)\right)^{1/2}$. Suppose that the kernel $\kappa$ is Hilbert-Schmidt, i.e., $\int_{\mathcal{X}} \int_{\mathcal{X}} \kappa^2\left(x,z\right) d\mu\left(x\right) d\mu\left(z\right) < \infty$, the Mercer's theorem \citep{Mercer09} states
\begin{align}
\label{eq:mercer}
	\kappa\left(x,z\right) = \sum_{i=1}^\infty \lambda_i e_i\left(x\right) e_i\left(z\right),
\end{align}
where convergence is absolute and uniform for a non-increasing sequence of eigenvalues $\left\{\lambda_i\right\}$ and the corresponding eigenfunctions $\left\{e_i\right\}$. These eigenfunctions are orthonormal and satisfy $\int_{\mathcal{X}} \kappa\left(x,z\right) e_i\left(z\right) d\mu\left(z\right) = \lambda_i e_i\left(x\right)$. From \eqref{eq:mercer}, there is an isomorphism between $\mathcal{H}$ and $\mathcal{G}$ under the linear map $\mathscr{I} \phi\left(x\right) \coloneqq \sum_{i=1}^\infty \sqrt{\lambda_i}e_i\left(x\right) \left(\sqrt{\lambda_i}e_i\right)$. Thus, $\mathcal{G}$ consists of functions $\sum_{i=1}^\infty w_i \sqrt{\lambda_i}e_i$ for $w_i \in \mathbb{R}$. In particular, one can compute the RKHS inner product using \eqref{eq:mercer} as $\left<\sum_{i=1}^\infty a_i \sqrt{\lambda_i}e_i, \sum_{i=1}^\infty b_i \sqrt{\lambda_i}e_i\right>_{\mathcal{H}} = \sum_{i=1}^\infty a_i b_i$. For any function $\sum_{i=1}^\infty w_i \sqrt{\lambda_i}e_i$ in $\mathcal{H}$, it holds that $\sum_{i=1}^\infty w_i^2 < \infty$. This yields the following Mercer's representation of $\mathcal{H}$:
\begin{align}
\label{eq:mercer-rep}
\mathcal{G} = \left\{\sum_{i=1}^\infty w_i\sqrt{\lambda_i}e_i \;\Big|\; \sum_{i=1}^\infty w_i^2 < \infty \right\}.
\end{align}
Denote by $a_i \coloneqq \int_{\mathcal{X}} f\left(x\right) e_i\left(x\right) d\mu\left(x\right)$, the constraint $\sum_{i=1}^\infty w_i^2 < \infty$ is alternatively written as $\sum_{i=1}^\infty a_i^2/\lambda_i < \infty$. The Mercer's representation \eqref{eq:mercer-rep} establishes a duality between the RKHS $\mathcal{H}$ and the corresponding $\mathcal{L}^2$ space $\mathcal{G}$ in which GPs are typically defined.

\subsubsection{Reproducing Kernel Hilbert Space of a Gaussian process}
\label{sec:GP-func}

We now recall the function perspectives of GPs, focusing on its connection to the RKHS. Let $\left(\Omega,\mathcal{F},P\right)$ be a probability space, where $\Omega$ is a sample space, $\mathcal{F}$ is an appropriate $\sigma$-algebra on $\Omega$, and $P$ is a probability measure. Denote by $\mathcal{L}^2\left(\Omega,\mathcal{F},P\right)$ the Hilbert space of real-valued square integrable random variables on $\Omega$:
\begin{align*}
\mathcal{L}^2\left(\Omega,\mathcal{F},P\right) \coloneqq
\left\{
Z: \Omega \mapsto \mathbb{R} :
\int_{\Omega} \left|Z\left(\omega\right)\right|^2 d P\left(\omega\right) < \infty
\right\}.
\end{align*}
A GP $f \in \mathcal{L}^2\left(\mathcal{X}\times \Omega,\mathcal{F},P\right)$ is a family of Gaussian random variables $\left\{f\left(x\right): x\in \mathcal{X}\right\}$ indexed by $\mathcal{X}$. Fixing some $\omega \in \Omega$, $\tilde{f}\left(\cdot\right) \coloneqq f\left(\cdot,\omega\right)$ is called a realization or a sample path of the GP. We will refer to the smallest Hilbert subspace of $\mathcal{L}^2\left(\Omega,\mathcal{F},P\right)$ that contains all sample paths as the Hilbert space spanned by the GP $f$, or the {\em Hilbert space of $f$}. The Hilbert space of $f$ is related to an RKHS through the following Lo\`{e}ve's lemma.

\begin{lemma}[\citealp{Loeve48}]
Let $\left\{f\left(x\right): x \in \mathcal{X}\right\}$ be a centered second-order process with covariance $k_f\left(x,z\right) \coloneqq \mathbb{E} \left(f\left(x\right) f\left(z\right)\right)$. Then, the Hilbert space $\mathcal{H}_f$ of $f$ is congruent to the RKHS $\mathcal{H}_{k_f}$ under the isometry $\mathscr{L}: \mathcal{H}_f \mapsto \mathcal{H}_{k_f}$ given by
\begin{align*}
\left(\mathscr{L} g\right)\left(x\right) \coloneqq \mathbb{E}\left(g f\left(x\right)\right), \quad g \in \mathcal{H}_f,\; x \in \mathcal{X}.
\end{align*}
\end{lemma}

The function space of GPs can be characterized via the duality specified by the Mercer's representation \eqref{eq:mercer-rep} which states that a function is contained in RKHS $\mathcal{H}$ if and only if it can be expressed as $f = \sum_{i=1}^\infty a_i e_i \in \mathcal{G}$ for some $\left\{a_i\right\}$ satisfying $\sum_{i=1}^\infty a_i^2/\lambda_i < \infty$. Suppose that $\left\{f\left(z\right) \mid z \in \mathcal{X}\right\}$ is a zero-mean GP with Hilbert-Schmidt covariance kernel $\kappa\left(x,z\right) \coloneqq \mathbb{E}\left(f\left(x\right) f\left(z\right)\right)$. Then it follows from \eqref{eq:mercer} that $\mathbb{E} a_i^2 = \lambda_i$. This is formalized by the Karhunen-Lo\`{e}ve expansion \citep{Ghanem91} which states that $f\left(x\right)$ has the following expansion:
\begin{align}
\label{eq:KLE}
f\left(x\right) = \sum_{i=1}^\infty \alpha_i e_i\left(x\right) \in \mathcal{L}^2\left(\Omega,\mathcal{F},P\right)
, \qquad  \alpha_i \coloneqq \int_{\mathcal{X}} f\left(x\right) e_i\left(x\right) d P\left(x\right)
\end{align}
and the coefficients $a_i$ are mutually uncorrelated random variables with
\begin{align}
\label{eq:KLE-a}
\mathbb{E} \alpha_i = 0, \qquad \mathbb{E} \alpha_i^2 = \lambda_i,
\end{align}
where convergence is in quadratic mean.

For Gaussian processes, $\alpha_i$ are independent and distributed as $\alpha_i \sim \mathcal{N}\left(0,\lambda_i\right)$. Often times, the eigenfunctions $e_i$ are difficult to compute and may not have closed analytic expressions. One can alternatively use a few basis functions $\psi_i: \mathcal{X} \mapsto \mathbb{R}$ in $\mathcal{G}$, and this leads to the function-space view of GPs \citep{Williams97}:
\begin{align}
	\label{eq:func-view}
    f\left(x\right) = \sum_{i=1}^\infty \alpha_i \psi_i\left(x\right), \qquad
    \bm{\alpha} \sim \mathcal{N}\left(\bm{0},\bm{\Sigma}_\alpha\right),
\end{align}
where we slightly abused the notation to write $\alpha_i$ for the coefficients associated with $\psi_i$.
The covariance kernel under the above GP specification is given by
\begin{align}
\label{eq:GP-var}
\text{Cov}\left( f\left(x\right), f\left(z\right) \right)
= \bm{\psi}_x^\top \bm{\Sigma}_\alpha \bm{\psi}_z, \qquad
\bm{\psi}_x = \left(\psi_1\left(x\right), \psi_2\left(x\right), \cdots \right)^\top.
\end{align}

\subsection{Proof of Theorems}
\label{sec:proofs}

We now provide the missing proofs in \cref{sec:func-spc} and \cref{sec:approx}. Recall that \cref{prop:func-inf} states that as $p$ increases the space of functions realizable by the IGP shrinks. This holds in general for a GP $\nu$ in either an RKHS generated by a kernel $\kappa_q$, $q\geq 0$, or the $\mathcal{L}^2\left(\mathcal{X}\times\Omega,\mathcal{F},P\right)$, as shown by the following simple \cref{prop:func}.

\begin{prop}
\label{prop:func}
Let $\left\{e_i\right\}$ and $\left\{\lambda_i\right\}$ be the eigenfunctions and eigenvalues of the kernel $\kappa$. For any function $h \in \left\{\sum_{i=1}^\infty w_i e_i : \sum_{i=1}^\infty w_i^2/\lambda_i^{2 p + q} < \infty\right\}$ and $q \geq 0$, the curve $h\left(\cdot\right)$ is a sample path of the IGP \eqref{eq:IGP-def} with kernel $\kappa$ and some $\nu \in \mathcal{H}_{\kappa_q}$.
\end{prop}
\begin{proof}
By Mercer's representation \eqref{eq:mercer-rep} of the RKHS, any sample path $h$ of the IGP admits the representation $h = \sum_{i=1}^\infty a_i e_i$ with $a_i = \int_{\mathcal{X}} h\left(x\right) e_i\left(x\right) d \mu\left(x\right)$. Similarly, we can also express $\nu = \sum_{i=1}^\infty b_i e_i$ with $\sum_{i=1}^\infty b_i^2 / \lambda_i^q < \infty$. Now rewrite $h$ in the form of \eqref{eq:IGP-def}, we have $\varphi_i = a_i/\lambda_i^p$. Together with the condition $\nu \in \mathcal{H}_{\kappa_q}$, we obtain $\sum_{i=1}^\infty \varphi_i^2/\lambda_i^q = \sum_{i=1}^\infty a_i^2/\lambda_i^{2 p + q} < \infty$, completing the proof.
\end{proof}
Note that \cref{prop:func-inf} is the special case of \cref{prop:func} with $q=0$ where $\mathcal{H}_{\kappa_q}$ is equivalent to the $\mathcal{L}^2$ for a non-degenerate kernel $\kappa$.

As for the proof of \cref{thm:IGP-sample-inf}, we first construct an RKHS $\mathcal{H}_{\kappa_{\mathcal{M}}^\prime}$ that contains $\mathcal{H}_{\kappa_{\mathcal{M}}}$ using \cref{lem:rkhs-incl}. Then, we use the dominance operator argument in \cref{lem:suff} to show that sample paths of $\mathcal{H}_{\kappa_{\mathcal{M}}^\prime}$ are contained in $\mathcal{H}_{\kappa}$. 

\begin{lemma}[\citealp{Aronszajn50}]
\label{lem:rkhs-incl}
$\mathcal{H}_{\kappa^\prime} \subset \mathcal{H}_{\kappa}$ if there is a constant $C < \infty$ such that $C^2 \kappa - \kappa^\prime$ is a nonnegative kernel.
\end{lemma}

\begin{lemma}[\citealp{Lukic01}]
\label{lem:suff}
Let $\left\{f\left(x\right): x\in \mathcal{X}\right\}$ be a Gaussian process with covariance kernel $k_f$ and let $\mathcal{H}_{\kappa}$ be an RKHS with reproducing kernel $\kappa$. Then, a sufficient and necessary condition for $P\left(f\left(\cdot\right) \in \mathcal{H}_{\kappa}\right) = 1$ is that there exists a trace class dominance operator $\mathscr{T}:\mathcal{H}_{\kappa} \mapsto \mathcal{H}_{\kappa}$ with range contained in $\mathcal{H}_f$ and satisfies
\begin{align*}
\left(\forall f \in \mathcal{H}_{\kappa}, \; \forall g \in \mathcal{H}_f\right) \quad
\left<f,g\right>_{\mathcal{H}_{\kappa}} = \left<\mathscr{T}f, g\right>_{\mathcal{H}_f}.
\end{align*}
\end{lemma}

Note that if the kernel operator of $\kappa_{\nu}$ has a finite number of nonzero eigenvalues, i.e., $\kappa_\nu$ is degenerate, $\kappa_{\mathcal{M}}$ will also be degenerate and $\mathcal{H}_{\kappa_{\mathcal{M}}} \subset \mathcal{H}_\kappa$. Thus, we only need to consider the case where $\kappa_{\nu}$ is non-degenerate. The following theorem provides a sufficient condition for the IGP sample paths to be contained in $\mathcal{H}_\kappa$.

\begin{theorem}
\label{thm:IGP-sample}
Let $f \in \mathcal{L}^2\left(\mathcal{X}\times\Omega,\mathcal{F},P\right)$ be an IGP specified by prior $\gamma\left(f; \kappa_p,\kappa_{\nu}\right)$. It holds that $P\left(f\left(\cdot\right) \in \mathcal{H}_{\kappa}\right) = 1$ if the kernel $\kappa_{2p - 1}$ is of trace class.
\end{theorem}
\begin{proof}
First, we show that the mean \eqref{eq:IGP-mean} is contained in $\mathcal{H}_{\kappa}$. Denote by $\left\{e_i\right\}$ the complete set of orthonormal eigenfunctions of $\kappa$ corresponding the non-increasing sequence of eigenvalues $\left\{\lambda_i\right\}$, and let $a_i = \int_{\mathcal{X}} \left(\mathbb{E} f\left(x\right)\right) e_i\left(x\right) d \mu\left(x\right)$. We have
\begin{align*} 
a_i^2 = \left(\lambda_i^p \int_{\mathcal{X}} \left[\mathbb{E} \nu\left(x\right)\right] e_i\left(x\right) d \mu\left(x\right)\right)^2
= \lambda_i^{2 p} \left<\mathbb{E} \nu, e_i\right>_{\mathcal{L}^2}^2 \leq \lambda_i^{2 p} \mathbb{E} \left<\nu, e_i\right>_{\mathcal{L}^2}^2,
\end{align*}
where the last inequality follows from Jensen's inequality. From Mercer's representation \eqref{eq:mercer-rep}, the squared RKHS norm of the mean function is given by $\left\|\mathbb{E} f\right\|_{\mathcal{H}_\kappa}^2 = \sum_{i=1}^\infty a_i^2/\lambda_i \leq \sum_{i=1}^\infty \lambda_1^{2 p - 1} \mathbb{E} \left<\nu, e_i\right>_{\mathcal{L}^2}^2$. By Parseval's identity, $\sum_{i=1}^\infty \lambda_1^{2 p - 1}\mathbb{E} \left<\nu, e_i\right>_{\mathcal{L}^2}^2 = \lambda_1^{2 p - 1}\mathbb{E} \left\|\nu\right\|_{\mathcal{L}^2}^2 < \infty$, yielding $\mathbb{E} f \in \mathcal{H}_{\kappa}$ as desired.

Let $\lambda_1^\prime$ denote the largest eigenvalue of the kernel operator of $\kappa_{\nu}$, it is easy to verify that $\kappa_{\mathcal{M}}^\prime\left(x,z\right) \coloneqq \lambda_1^\prime\sum_{i=1}^\infty \lambda_i^{2 p} e_i\left(x\right) e_i\left(z\right)$ is a kernel that satisfies $\kappa_{\mathcal{M}}^\prime - \kappa_{\mathcal{M}}$ is nonnegative definite. By \cref{lem:rkhs-incl}, we then have $\mathcal{H}_{\kappa_{\mathcal{M}}} \subset \mathcal{H}_{\kappa_{\mathcal{M}}^\prime}$.

Now we use the dominance operator in \cref{lem:suff} to show that sample paths of $\mathcal{H}_{\kappa_{\mathcal{M}}^\prime}$ are contained in $\mathcal{H}_{\kappa}$. Denote by $\mathscr{T}: \mathcal{H}_{\kappa} \mapsto \mathcal{H}_{\kappa}$ be the dominance operator. From Mercer's representation \eqref{eq:mercer-rep}, the complete orthonormal basis of $\mathcal{H}_{\kappa}$ and $\mathcal{H}_{\kappa_{\mathcal{M}}^\prime}$ are respectively $\left\{\sqrt{\lambda_i} e_i\right\}$ and $\left\{\sqrt{\lambda_1^\prime}\lambda_i^p e_i\right\}$. Recall that the dominance operator satisfies
\begin{align*}
\left(\forall i \in \mathbb{N}\right) \quad
\left<\mathscr{T} \sqrt{\lambda_i} e_i, \sqrt{\lambda_1^\prime}\lambda_i^p e_i\right>_{\mathcal{H}_{\kappa_{\mathcal{M}}^\prime}}
= \left<\sqrt{\lambda_i} e_i,\sqrt{\lambda_1^\prime}\lambda_i^p e_i\right>_{\mathcal{H}_{\kappa}}
= \sqrt{\lambda_1^\prime}\lambda_i^{p-\frac{1}{2}}.
\end{align*}
We obtain $\mathscr{T} \sqrt{\lambda_i} e_i = \lambda_1^\prime \lambda_i^{2p-1} \sqrt{\lambda_i} e_i$, for all $i \in \mathbb{N}$. Thus, the trace of $\mathscr{T}$ is given by $\sum_{i=1}^\infty \left<\mathscr{T} \sqrt{\lambda_i} e_i,\sqrt{\lambda_i} e_i\right>_{\mathcal{H}_{\kappa}} = \lambda_1^\prime \sum_{i=1}^\infty \lambda_i^{2p-1} < \infty$, which holds for trace class $\kappa_{2p - 1}$. Invoking \cref{lem:suff} completes the proof.
\end{proof}

\cref{thm:sdr-rank} states that if the reproducing kernel $\kappa$ of the SIGP function space $\mathcal{H}_{\kappa,n}$ has fast eigenvalue decay, then $\mathcal{H}_{\kappa,n}$ can be well-approximated by a low-rank SDR subspace, and hence the approximate SIGP \eqref{eq:apx-SIGP} has a low-rank covariance.

\begin{theorem}
\label{thm:sdr-rank}
Under the assumption \eqref{eq:sdr-dr} and let $\tau_1\geq\tau_2\geq\cdots\geq\tau_n$ be the eigenvalues for $\text{Var}\left(\mathbb{E}\left(Y\mid X\right)\right) \bm{S}_i = \tau_i\text{Var}\left(X\right) \bm{S}_i$. Then, with probability at least $1-\delta$ the true SDR rank $m_{\star}$ for the model \eqref{eq:lr} with the SIGP prior \eqref{eq:SIGP} satisfies
\begin{align}
\tau_{m_{\star}} \geq \frac{1}{n} - \sqrt{\frac{8}{n^3}\log \frac{2}{\delta}}. \label{eq:sdr-rank}
\end{align}
\end{theorem}

\begin{proof}
We first start with a result (Theorem 5.1 of \citealp{Li91}) which states that $n \sum_{i=m_\star+1}^n \tau_i$ follows a chi-squared distribution with $\left(n-m_\star\right)\left(s - m_\star - 1\right)$ degree of freedom, where $s$ is the number of slices used for computing the sample $\text{Var}\left(\mathbb{E}\left(Y\mid X\right)\right)$. Set $s = m_\star + 2$ to get $n \sum_{i=m_\star+1}^n \tau_i \sim \chi_{n-m_\star}^2$. It can be easily shown that $\chi_k^2$ is sub-exponential with the tail bound 
\begin{align*}
P\left(\left|\chi_k^2/k - 1\right|\geq t\right) \leq 2\exp\left(- n t^2/8\right), \quad \forall 0 < t < 1.
\end{align*}
Denote by $\overline{\tau} = \sum_{i=m_\star+1}^n \tau_i / \left(n-m_\star\right)$, we obtain from the tail bound that
\begin{align}
\left|\overline{\tau} - \frac{1}{n}\right| \leq \frac{1}{n} \sqrt{\frac{8}{n}\log\frac{2}{\delta}}
\label{eq:rank-bound} 
\end{align}
holds with probability at least $1-\delta$. From \eqref{eq:rank-bound} and note that $\tau_{m_\star} \geq \overline{\tau}$, we arrive at \eqref{eq:sdr-rank}.
\end{proof}

\subsection{The Likelihood over SDR Subspaces of $\mathcal{H}_{\kappa,n}$}
\label{sec:theory-sdr}

In this section, we derive the likelihood over SDR subspaces of the RKHS. Our point of departure is provided by \cref{lem:lad} proposed in \citep{Cook09} which specifies the likelihood over SDR subspace in $\mathbb{R}^d$. \cite{Cook09} showed that the likelihood achieves more robust SDR subspace estimation compared to the generalized eigenvalue decomposition implied by \cref{thm:sir} which essentially yields the intersection of SDR subspaces known as the central subspace \citep{Cook98}.

\begin{lemma}[\citealp{Cook09}]
\label{lem:lad}
Consider the regression problem with the response $Y \in \mathbb{R}$ and covariates
$X \in \mathbb{R}^d$. If the conditional distribution $X \mid Y$ is normal, then
the log-likelihood for the SDR subspace spanned by the columns of $\bm{S}$ is
given by
\begin{align}
\label{eq:lad-lik}
\begin{split}
l\left(\bm{S}\right) \propto &-\frac{1}{2}
\sum_y n_y \log \det \left[\bm{S}^\top \text{Var}\left(X\mid y\right) \bm{S} \right]
 + \frac{n}{2} \log \det \left[\bm{S}^\top \text{Var}\left(X\right) \bm{S} \right]
\end{split}
\end{align}
where $y$ is the index over the slices of the range of $Y$, and $n_y$ denotes the number of data points in the slice.
\end{lemma}
Here we denote by $\bm{S}$ the SDR basis of $\mathbb{R}^d$ in constrast to the SDR basis $\bm{W}$ in the RKHS setting. Similar to \citep{Li91}, Cook's method also computes the sample $\text{Var}\left(X\mid Y\right)$ by slicing the range of $Y$, and $\text{Var}\left(X\mid y\right)$ in \cref{lem:lad} represents the sample variance of $X$ within the $y$-th slice.

We show in \cref{thm:equi} an equivalent form of the Cook's likelihood \eqref{eq:lad-lik}, but with an explicit maximizer. We first give the likelihood over SDR subspaces of $\mathbb{R}^d$, and then extend the result to the RKHS setting in \cref{sec:sdr-rkhs}.

\begin{theorem}
\label{thm:equi}
Under the same conditions of \cref{lem:lad}, the likelihood for the SDR subspace with basis $\bm{S} \in \mathbb{R}^{d\times m}$ is written
\begin{align}
\label{eq:sdr-lik}
g\left(\bm{S}\right) \propto -
\frac{n}{2} \log \frac{\det \left[\bm{S}^\top \mathbb{E}\left(\text{Var}\left(X \mid Y\right)\right) \bm{S}\right] } { \det \left(\bm{S}^\top \text{Var}\left(X\right) \bm{S}\right) },
\end{align}
\end{theorem}
\begin{proof}
We will use the SDR subspace characterization given by Proposition 1(i) of \citep{Cook09}, which states
\begin{align}
\label{eq:sdr-nonrand}
\bm{S}_\bot^\top \text{Var}^{-1}\left(X\mid Y\right) = \bm{S}_\bot^\top \mathbb{E}^{-1}\left(\text{Var}\left(X\mid Y\right)\right),
\end{align}
where $\bm{S}_\bot$ denotes the orthogonal complement of $\bm{S}$. It is clear that $g\left(\bm{S}\right)$ is invariant to linear transformations of $\bm{S}$, thus we can assume without loss of generality that $\bm{S}$ is semi-orthogonal.

The following identity \eqref{eq:rao} is from \citep{Rao73}:
Let $\bm{A} \in \mathbb{R}^{p \times n}$ be of rank $n$ and let $\bm{S} \in \mathbb{R}^{p \times \left(p-n\right)}$ be of rank $p - n$ such that $\bm{A}^\top \bm{B} = \bm{0}$. Then
\begin{align}
\label{eq:rao}
\bm{\Sigma} = \bm{B}\left(\bm{B}^\top\bm{\Sigma}^{-1}\bm{B}\right)^{-1}\bm{B}^\top
+ \bm{\Sigma}\bm{A}\left(\bm{A}^\top\bm{\Sigma}\bm{A}\right)^{-1}\bm{A}^\top\bm{\Sigma}.
\end{align} 
Note that if both $\bm{A}$ and $\bm{B}$ are semi-orthogonal, then due to \eqref{eq:rao}:
\begin{align*}
\bm{S}^\top\bm{\Sigma}\bm{S} - \bm{S}^\top\bm{\Sigma}\bm{S}_\bot\left(\bm{S}_\bot^\top\bm{\Sigma}\bm{S}_\bot\right)^{-1}\bm{S}_\bot^\top\bm{\Sigma}\bm{S} = \left(\bm{S}^\top\bm{\Sigma}^{-1}\bm{S}\right)^{-1}.
\end{align*}
Observe that the left hand side is the Schur's complement. We then rewrite $\det\left(\bm{\Sigma}\right)$ as
\begin{align*}
\det\left(\begin{bmatrix} \bm{S} & \bm{S}_\bot \end{bmatrix}^\top \bm{\Sigma} 
\begin{bmatrix} \bm{S} & \bm{S}_\bot \end{bmatrix} \right)
= \det\left( \begin{bmatrix} \bm{S}^\top\bm{\Sigma}\bm{S} & \bm{S}^\top\bm{\Sigma}\bm{S}_\bot\\
\bm{S}_\bot^\top\bm{\Sigma}\bm{S} & \bm{S}_\bot^\top\bm{\Sigma}\bm{S}_\bot
\end{bmatrix} \right) = \det\left( \bm{S}_\bot^\top\bm{\Sigma}\bm{S}_\bot \right) \det\left(\bm{S}_C\right),
\end{align*}
where $\bm{S}_C \coloneqq \bm{S}^\top\bm{\Sigma}\bm{S} - \bm{S}^\top\bm{\Sigma}\bm{S}_\bot \left(\bm{S}_\bot^\top\bm{\Sigma}\bm{S}_\bot\right)^{-1} \bm{S}_\bot^\top\bm{\Sigma}\bm{S}$ is the Schur's complement. We obtain $\det\left( \bm{S}_\bot^\top\bm{\Sigma}\bm{S}_\bot \right) = \det\left( \bm{\Sigma} \right)\det\left(\bm{S}^\top\bm{\Sigma}^{-1}\bm{S}\right)$. Together with \eqref{eq:sdr-nonrand}, we have
\begin{align}
\label{eq:det}
\frac{\det\left[\bm{S}^\top\text{Var}\left(X \mid y\right)\bm{S}\right]}{\det\left[\bm{S}^\top \mathbb{E}\left(\text{Var}\left(X \mid Y\right)\right) \bm{S}\right]} = 
\det\left[\mathbb{E}\left(\text{Var}\left(X \mid Y\right)\right)\right]
\det\left[\text{Var}\left(X \mid y\right)\right].
\end{align}
Note that $\sum_y n_y \log \det \left[\bm{S}^\top \text{Var}\left(X\mid y\right) \bm{S} \right] = n\mathbb{E}\left\{\log \det \left[\bm{S}^\top \text{Var}\left(X\mid y\right) \bm{S} \right]\right\}$ in \eqref{eq:lad-lik}. Also, \eqref{eq:det} yields
\begin{align*}
\mathbb{E}\left\{\log\det\left[\bm{S}^\top\text{Var}\left(X \mid y\right)\bm{S}\right]\right\} 
&= \log\det\left[\bm{S}^\top \mathbb{E}\left(\text{Var}\left(X \mid Y\right)\right) \bm{S}\right] 
+ \log\det\left[\mathbb{E}\left(\text{Var}\left(X \mid Y\right)\right)\right]\\
&\phantom{=}+ \mathbb{E}\left\{\log\det\left[\text{Var}\left(X \mid Y\right)\right]\right\}.
\end{align*}
The first term on the right is the numerator in \eqref{eq:sdr-lik}, and the last two terms do not involve $\bm{S}$. This shows that $g\left(\bm{S}\right) - l\left(\bm{S}\right)$ is a constant. By \cref{lem:lad}, $g\left(\bm{S}\right)$ gives unbiased estimation of the SDR subspace.
\end{proof}

The utility of \cref{thm:equi} is that the log-likelihood \eqref{eq:sdr-lik} has an explicit maximizer $\bm{S}_\star$ whose columns are the leading eigenvectors of $\mathbb{E}^{-1}\left(\text{Var}\left(X \mid Y\right)\right) \text{Var}\left(X\right)$. This is an immediate consequence of the following \cref{lem:det-quotient}. 

\begin{lemma}
\label{lem:det-quotient}
For positive definite matrices $\bm{M},\bm{N} \in \mathbb{R}^{n \times n}$, the column space of an optimal full-rank $\bm{S}_\star \in \mathbb{R}^{n \times m}$, $m \leq n$, for
\begin{align}
\label{eq:min-det-quo} 
\min_{\bm{S}} \frac{\det\left(\bm{S}^\top\bm{M}\bm{S}\right)}{\det\left(\bm{S}^\top\bm{N}\bm{S}\right)}
\end{align}
coincides with the span of the $d$ leading eigenvectors of $\bm{M}^{-1}\bm{N}$.
\end{lemma}
\begin{proof}
Since $\bm{M}$ is positive definite, we denote by $\bm{S} = \bm{M}^{-1/2}\bm{T}$ and let $\bm{T} = \bm{Q}\bm{R}$ be the QR decomposition. Observe that both $\bm{S}$ and $\bm{R}$ are of full rank. The objective of \eqref{eq:min-det-quo} can be rewritten as
\begin{align*}
\frac{1}{\det\left(\bm{Q}^\top\bm{M}^{-1/2}\bm{N}\bm{M}^{-1/2}\bm{Q}\right)}.
\end{align*}
The minimum is attained by setting the columns of $\bm{Q}$ to the leading eigenvectors of $\bm{M}^{-1/2}\bm{N}\bm{M}^{-1/2}$. Then, the columns of $\bm{M}^{-1/2}\bm{Q}$ are the leading eigenvectors of $\bm{M}^{-1}\bm{N}$, and the column space of $\bm{S} =\bm{M}^{-1/2}\bm{T} = \left(\bm{M}^{-1/2}\bm{Q}\right)\bm{R}$ is the same as the eigenspace spanned by the leading eigenvectors of $\bm{M}^{-1}\bm{N}$.
\end{proof}

The log-likelihood \eqref{eq:sdr-lik} has several desirable properties. First, $g\left(\bm{S}\right) = g\left(\bm{S}^\prime\right)$ whenever $\bm{S}$ and $\bm{S}^\prime$ are bases of the same subspace. Thus, the maximizer $\bm{S}_\star$ for \eqref{eq:sdr-lik} is not unique, but they span the same SDR subspace. Second, the log-likelihood yields the same $\bm{S}_\star$ as Fisher linear discriminant analysis (FDA) for two groups, i.e., $Y$ is binary, when $\bm{S}$ is a vector. In this case, 
\begin{align}
\label{eq:lda}
g\left(\bm{S}\right) \propto -\log\left(1 - \frac{\bm{S}^\top\text{Var}\left(\mathbb{E}\left(X\mid Y\right)\right)\bm{S}}{\bm{S}^\top \text{Var}\left(X\right) \bm{S}}\right),
\end{align}
where the fraction in \eqref{eq:lda} is the objective of FDA. Moreover, \cref{prop:sir} states that the SDR subspace basis given by the sliced inverse regression (SIR) \citep{Li91} are in the vector space spanned by a maximizer of \eqref{eq:sdr-lik}.

\begin{prop}
\label{prop:sir}
Suppose that $\mathbb{E}\left(\text{Var}\left(X\mid Y\right)\right)$ and $\text{Var}\left(X\right)$ are positive definite. Then, the column space of a maximizer $\bm{S}_\star$ to $g\left(\bm{S}\right)$ contains the central subspace estimated by the sliced inverse regression. 
\end{prop}
\begin{proof}
The SIR solves for the basis $\bm{b}_i$ of the SDR (central) subspace via the following generalized eigenvalue decomposition:
\begin{align}
\label{eq:sir-eig}
\text{Var}\left(\mathbb{E}\left(X \mid Y\right)\right)\bm{b}_i = \tau_i \text{Var}\left(X\right)\bm{b}_i.
\end{align}
The variance decomposition gives $\left[\text{Var}\left(X\right) - \mathbb{E}\left(\text{Var}\left(X \mid Y\right)\right)\right] \bm{b}_i = \tau_i \text{Var}\left(X\right) \bm{b}_i$. Rearrange the terms, we obtain
\begin{align}
\label{eq:lad-eig}
\left[\mathbb{E}\left(\text{Var}\left(X \mid Y\right)\right)\right]^{-1}\text{Var}\left(X\right) \bm{b}_i = \frac{1}{1-\tau_i} \bm{b}_i.
\end{align}
From \cref{thm:equi} and \cref{lem:det-quotient}, $\bm{S}_\star$ is specified by the eigenvectors of \eqref{eq:lad-eig}. Clearly, any eigenvector $\bm{b}_i$ of \eqref{eq:sir-eig} corresponding to a nonzero eigenvalue is an eigenvector of \eqref{eq:lad-eig}. However, when $\mathbb{E}\left(X\mid Y\right) \equiv 0$, \eqref{eq:sir-eig} fails to recover the SDR direction $\bm{b}_i$ corresponding to the $\tau_i = 0$ as reported in \citep{Cook91}. Note that we have $\tau_i < 1$ due to the variance decomposition. Thus, the central subspace basis given by SIR corresponding to the largest $\tau_i$ are also the leading eigenvectors of \eqref{eq:lad-eig}.
\end{proof}

\subsubsection{Extension to RKHS}
\label{sec:sdr-rkhs}
\cref{thm:equi} can be extended to the RKHS setting under an analogous \cref{cond:ldc-rkhs} to \cref{cond:ldc}. The condition is basically a restatement of \cref{cond:ldc} in terms of the basis $\kappa\left(\cdot,x_1\right),\cdots,\kappa\left(\cdot,x_n\right)$ of the finite-dimensional subspace of $\mathcal{H}_{\kappa}$. It has been shown that \cref{cond:ldc-rkhs} is fairly reasonable in the setting of cross-covariance operators on Hilbert spaces \citep{Fukumizu04}. Recall that the SDR subspace of the RKHS is spanned by \eqref{eq:basis}, we aim to estimate the basis coefficients $\bm{W}$ on a finite sample.

\begin{cond}
\label{cond:ldc-rkhs}
Suppose that $\kappa$ is of trace class. For any $f\left(\cdot\right) = \sum_{i=1}^n a_i \kappa\left(\cdot,x_i\right)$, $\left\{a_i\right\}_{i=1}^n \subset \mathbb{R}$, there exists $\left\{c_i\right\}_{i=0}^m \subset \mathbb{R}$ such that $\mathbb{E}\left(f\left(x\right) \mid \sum_{i=1}^n W_{i1}\kappa\left(x,x_i\right),\cdots,\sum_{i=1}^n W_{im}\kappa\left(x,x_i\right)\right) = c_0 + \sum_{j=1}^m c_j \sum_{i=1}^n W_{ij}\kappa\left(x,x_i\right)$ for all $x\in\mathcal{X}$.
\end{cond}

Denote by $\bm{\phi} \coloneqq \left(\kappa\left(\cdot,x_1\right),\cdots,\kappa\left(\cdot,x_n\right)\right)$ the feature vector on the sample $\left\{x_i\right\}_{i=1}^n$, and $\bar{\bm{\phi}} \coloneqq \left(\bar{\kappa}\left(\cdot,x_1\right),\cdots,\bar{\kappa}\left(\cdot,x_n\right)\right) = \bm{\phi}\bm{\Gamma}_n$ the centered feature vector with $\bar{\kappa}\left(\cdot,x_i\right)$ and $\bm{\Gamma}_n$ as defined in \eqref{eq:SIGP}. Observe that $\bm{K} = \bm{\phi}^\top \bm{\phi}$ and $\bm{\Gamma}_n$ is idempotent. In the following, we provide two methods for estimating $\bm{W}$.

First consider a slicing based estimation of $\bm{W}$ as in SIR. Suppose that the data $\mathcal{D}$ is sorted by the response value. Partition the data into slices $\left\{\left(x_1,y_1\right),\cdots,\left(x_{n_1},y_{n_1}\right)\right\}$, $\left\{\left(x_{n_1 + 1},y_{n_1+1}\right),\cdots,\left(x_{n_1+n_2},y_{n_1+n_2}\right)\right\}$, and so forth, where $n_i$ denotes the size of the $i$-th slice. We replace $X$ with $\bm{\phi}$ in the log-likelihood \eqref{eq:sdr-lik}, and compute $\mathbb{E}\left(\text{Var}\left(\bm{\phi} \mid Y\right)\right)$ as the weighted average of the slice sample variances. Specifically, we have
\begin{gather*}
\text{Var}\left(\bm{\phi}\right) = \frac{1}{n} \bar{\bm{\phi}} \bar{\bm{\phi}}^\top = \frac{1}{n} \bm{\phi} \bm{\Gamma}_n \bm{\phi}^\top\\ 
\mathbb{E}\left(\text{Var}\left(\bm{\phi} \mid Y\right)\right)
= \bm{\phi} \left[ \frac{1}{n}\bm{I}_n -
\frac{1}{n}\text{diag}\left(\frac{\bm{1}_{n_i} \bm{1}_{n_i}^\top}{n_i}\right)\right] \bm{\phi}^\top
= \frac{1}{n} \bm{\phi} \text{diag}\left(\bm{\Gamma}_{n_i}\right) \bm{\phi}^\top,
\end{gather*}
where $\text{diag}\left(\bm{\Gamma}_{n_i}\right)$ denotes the block diagonal matrix with diagonal blocks $\bm{\Gamma}_{n_i}$. Thus, the log-likelihood \eqref{eq:sdr-lik} over SDR subspaces of the RKHS is written
\begin{align}
g\left(\bm{W}\right) = - \frac{n}{2} \log \frac{\det \left( \bm{W}^\top \bm{M} \bm{W} \right)}
{\det \left(\bm{W}^\top \bm{N} \bm{W}\right)}
\label{eq:lik-rkhs}
\end{align}
with $\bm{M} \coloneqq \bm{K}\text{diag}\left(\bm{\Gamma}_{n_i}\right)\bm{K} + n \zeta \bm{K}$ and $\bm{N} \coloneqq \bm{K}\bm{\Gamma}_n\bm{K}$. Here, $\bm{M}$ is obtained by adding a small constant $\zeta \bm{I}_n$, $\zeta > 0$, to $\mathbb{E}\left(\text{Var}\left(\bm{\phi} \mid Y\right)\right)$ in order for the conditions of \cref{prop:sir} to hold. It is well-known that adding the constant imposes Tikhonov regularization on $\bm{W}$ from the Lagrange multiplier perspective. 

Under some mild conditions, the sample $\mathbb{E}\left(\text{Var}\left(\bm{\phi} \mid Y\right)\right)$ can be computed without slicing as stated in the following theorem.
\begin{theorem}[\citealp{Fukumizu04}]
\label{thm:cond-var}
Assume $\kappa$ is of trace class, and there exists a function $f_Y: \mathbb{R} \mapsto \mathbb{R}$ in an RKHS $\mathcal{H}_{\kappa_Y}$ with trace class reproducing kernel $\kappa_Y$ satisfying
$\mathbb{E}\left(f\left(x\right) \mid y\right) = f_Y\left(y\right)$ for any $f\in \mathcal{H}_{\kappa}$ and almost every $y$, then
\begin{align*}
\mathbb{E}\left(\text{Var}\left(\phi\left(X\right)\mid Y\right)\right) = \text{Var}\left(\phi\left(X\right)\right) - \text{Cov}\left(\phi\left(X\right),Y\right) \text{Var}^{-1}\left(Y\right) \text{Cov}\left(Y,\phi\left(X\right)\right).
\end{align*}
\end{theorem}
Let $\bm{K}_{\kappa_y}$ be the kernel matrix generated by the response $\left\{y_i\right\}_{i=1}^n$ and $\kappa_Y$. Denote by $\bm{\phi}_Y \coloneqq \left(\kappa_Y\left(\cdot,y_1\right),\cdots,\kappa_Y\left(\cdot,y_n\right)\right)$, then it follows from \cref{thm:cond-var}
\begin{align*}
\mathbb{E}\left(\text{Var}\left(\bm{\phi}\mid Y\right)\right) 
&= \frac{1}{n}\bm{\phi}\bm{\Gamma}_n\bm{\phi}^\top 
- \frac{1}{n^2}\bm{\phi}\bm{\Gamma}_n\bm{\phi}_Y^\top
\left(\frac{1}{n}\bm{\phi}_Y\bm{\Gamma}_n\bm{\phi}_Y^\top + \zeta_1\bm{I}\right)^{-1}
\bm{\phi}_Y\bm{\Gamma}_n\bm{\phi}^\top\\
&= \frac{1}{n}\bm{\phi}\bm{\Gamma}_n\bm{\phi}^\top 
- \frac{1}{n}\bm{\phi}\left(\bm{\Gamma}_n\bm{K}_{\kappa_Y}\bm{\Gamma}_n + n\zeta_1\bm{I}\right)^{-1}\bm{\Gamma}_n\bm{K}_{\kappa_Y}\bm{\Gamma}_n
\bm{\phi}^\top\\
&= \frac{1}{n}\bm{\phi}
\left[\bm{\Gamma}_n -  
\left(\overline{\bm{K}}_{\kappa_Y} + n\zeta_1\bm{I}\right)^{-1} \overline{\bm{K}}_{\kappa_Y}
\right] \bm{\phi}^\top,
\end{align*}
where $\overline{\bm{K}}_{\kappa_Y} \coloneqq \bm{\Gamma}_n\bm{K}_{\kappa_Y}\bm{\Gamma}_n$ represents the centered kernel matrix of $\bm{y}$ (see e.g., \citealp{Scholkopf98}), and $\zeta_1 > 0$ is a small constant added to ensure the positive definiteness of $\text{Var}\left(Y\right)$. Note that $\bm{\Gamma}_n$ is symmetric as well as idempotent in deriving the second equality.

\section{Algorithms for Parameter Inference}
\label{sec:alg}

In this section, we develop fast algorithms for estimating the parameters of the SIGP \eqref{eq:apx-SIGP}. Learning the IGP \eqref{eq:IGP-def} can be carried out similarly as in training the standard GP by maximizing the marginal likelihood \eqref{eq:posterior} \citep{Rasmussen06}. 

The estimation task is to infer parameters of both the mean function and \eqref{eq:apx-SIGP}, namely the SDR basis $\bm{W}$, $\bm{\Sigma}_\beta$, as well as $\sigma^2$. Specifically, the estimation of $\bm{W}$ uses the log-likelihood \eqref{eq:lik-rkhs-inf}, and can be computed directly by \cref{lem:det-quotient}. In addition, inference of the covariance parameters leverages the low-rank SDR parameterization via an Expectation-Maximization (EM) algorithm. While $\kappa$ is assumed to be given and fixed, inferring the hyper-parameters of $\kappa$ is also possible.

\subsection{Selecting the Sufficient Dimension Reduction Subspace}

Recall the log-likelihood \eqref{eq:lik-rkhs-inf} of the SDR subspace basis $\bm{W}$, whose maximizer from \cref{lem:det-quotient} is given by the leading eigenvectors for the generalized eigenvalue decomposition
\begin{align}
\label{eq:rkhs-eig}
\bm{\Gamma}_n\bm{K} \bm{W}_i = \tau_i \left(\bm{A} + n \zeta\bm{I}_n\right) \bm{W}_i,
\end{align}
where $\bm{\Gamma}_n$ is as defined in \eqref{eq:SIGP}, $\zeta > 0$ is a regularization parameter, and $\bm{A}$ is computed from the kernel matrix $\bm{K}$ discussed next.

As we discussed in \cref{sec:sdr-rkhs}, there are two ways for computing $\bm{A}$: 1) a fast slicing-based approach as in the sliced inverse regression; and 2) the method that uses an additional kernel of the response $Y$ which potentially yields improved estimates in regression by exploiting the local information in $Y$. 

For the slicing-based approach, first partition the data into slices $\left\{\left(x_1,y_1\right),\cdots,\left(x_{n_1},y_{n_1}\right)\right\}$, $\left\{\left(x_{n_1 + 1},y_{n_1+1}\right),\cdots,\left(x_{n_1+n_2},y_{n_1+n_2}\right)\right\}$, and so forth, where $n_i$ denotes the size of the $i$-th slice. Then, let $\bm{A} = \text{diag}\left(\bm{\Gamma}_{n_i}\right)\bm{K}$, and solve the generalized eigenvalue decomposition \eqref{eq:rkhs-eig}. The overall computational complexity in this case is $O\left(n^2 m\right)$. As for the other method, the sorting is not needed, but instead a kernel matrix $\overline{\bm{K}}_{\kappa_Y}$ for $Y$ is needed. We will instead let $\bm{A} = \bm{\Gamma}_n\bm{K} - \left(\overline{\bm{K}}_{\kappa_Y} + n\zeta_1\bm{I}\right)^{-1} \overline{\bm{K}}_{\kappa_Y}\bm{K}$ in \eqref{eq:rkhs-eig}, where $\zeta_1 > 0$ is a constant for the inverse to be well-defined (see also \citealp{Fukumizu04}). Clearly, the computational complexity for solving the slicing-free SDR estimation is $O\left(n^3\right)$.  

\subsection{Estimating the Covariance via Expectation Maximization}
The SIGP regression model \eqref{eq:lr} with the approximate SIGP prior \eqref{eq:apx-SIGP} can be viewed as a latent variable model $Y\left(\cdot\right) \sim u\left(\cdot\right) + \bm{\Pi}\left(\cdot\right)\bm{\beta} + \epsilon$, where $u$ is the mean function and $\bm{\beta} \sim \mathcal{N}\left(\bm{0},\bm{\Sigma}_\beta\right)$ is a latent vector. Intuitively, we may consider an EM algorithm for estimating the variance components $\bm{\Sigma}_\beta$ as well as $\sigma^2$. Parameter inference for the mean can also be performed during the M-step. For ease of exposition, we will use the mean function $u\left(\cdot\right) = \bm{\Pi}\left(\cdot\right)\bm{\alpha} + c$ with parameters $\bm{\alpha}$ and $c$.

First, the log-likelihood of the SIGP regression model after dropping irrelevant terms is written
\begin{align}
\label{eq:log-lik}
\begin{split}
l\left(\bm{\Sigma}_\beta,\sigma^2\right)
\propto -\frac{1}{2}\log\det \bm{\Sigma}_\beta
- \frac{n}{2} \log \sigma^2
- \frac{1}{2}\bm{\beta}^\top \bm{\Sigma}_\beta ^{-1} \bm{\beta}
- \frac{\bm{\epsilon}^\top\bm{\epsilon}}{2 \sigma^2}.
\end{split}
\end{align}
For notational convenience, denote by $\bm{\Pi} \coloneqq \bm{\Pi}\left(\bm{X}\right) = \bm{\Gamma}_n \bm{K}\bm{W}$. From the latent variable model view of the SIGP, we obtain the posterior distribution of $\bm{\beta}$
\begin{align}
\label{eq:cond-beta}
\begin{split}
\mathcal{N}\left(
\widehat{\bm{\Sigma}}_\beta\bm{\Pi}^\top \widehat{\bm{V}}^{-1} \left(\bm{y} - u\left(\bm{X}\right)\right),
\widehat{\bm{\Sigma}}_\beta - \widehat{\bm{\Sigma}}_\beta\bm{\Pi}^\top \widehat{\bm{V}}^{-1} \bm{\Pi}\widehat{\bm{\Sigma}}_\beta\right).
\end{split}
\end{align}
The above posterior distribution gives the MAP estimator $\widehat{\bm{\beta}}$ for the latent variable, which will be used in the M-step. Let $\widehat{\bm{r}} \coloneqq \bm{y} - u\left(\bm{X}\right)$ denote the random effect, the MAP estimator $\widehat{\bm{\beta}}$ can be equivalently expressed as
\begin{align}
\label{eq:beta-opt}
\widehat{\bm{\beta}} = \left(
\widehat{\sigma}^2 \widehat{\bm{\Sigma}}_\beta^{-1} + \bm{\Pi}^\top\bm{\Pi}
\right)^{-1} \bm{\Pi}^\top \widehat{\bm{r}} 
= \widehat{\sigma}^{-2} \widehat{\bm{\Delta}} \bm{\Pi}^\top \widehat{\bm{r}} 
\end{align}
where $\widehat{\bm{\Delta}}$ is as defined in \eqref{eq:pred-var}.

By taking the expectation of \eqref{eq:log-lik} with respect to the posterior distribution \eqref{eq:cond-beta}, we arrive at the E-step:
\begin{align*} 
\mathbb{E}_{\bm{\beta} \mid \mathcal{D}} \left[l\left(\bm{\Sigma}_\beta,\sigma^2\right)\right]
&\propto -\frac{1}{2}\log\det\bm{\Sigma}_\beta -\frac{n}{2} \log \sigma^2 -\frac{1}{2} \widehat{\bm{\beta}}^\top\bm{\Sigma}_\beta^{-1}\widehat{\bm{\beta}} -\frac{1}{2} \tr \widehat{\bm{\Delta}}\bm{\Sigma}_\beta^{-1}\\
&\phantom{=} - \frac{1}{2 \sigma^2} \left[
\left\|\widehat{\bm{r}} - \bm{\Pi}\widehat{\bm{\beta}}\right\|^2
+ \tr \left(\bm{\Pi}\widehat{\bm{\Delta}}\bm{\Pi}^\top\right)
\right],
\end{align*}
where we used the fact that the covariance of \eqref{eq:cond-beta} satisfies $\widehat{\bm{\Sigma}}_\beta - \widehat{\bm{\Sigma}}_\beta\bm{\Pi}^\top\widehat{\bm{V}}^{-1}\bm{\Pi}\widehat{\bm{\Sigma}}_\beta 
= \left(\widehat{\bm{\Sigma}}_\beta^{-1} 
+ \widehat{\sigma}^{-2}\bm{\Pi}^\top \bm{\Pi}\right)^{-1}
= \widehat{\bm{\Delta}}$. Also note that $\widehat{\bm{\beta}}$ and $\widehat{\bm{\Delta}}$ are computed based on the estimated $\widehat{\bm{\Sigma}}_\beta$ and $\widehat{\sigma}^2$ from the previous iteration.

\begin{figure}[t]
  \centering
  \begin{minipage}{.82\linewidth}
\begin{algorithm}[H]
\SetAlgoLined
\SetNlSty{texttt}{[}{]}
\KwIn {$\bm{K}$, $\bm{y}$, SDR rank $m$ as well as regularization parameter $\zeta$, $\xi$}
\KwOut {$\widehat{\bm{\alpha}}$, $\widehat{c}$, $\widehat{\bm{\Sigma}}_{\beta}$, $\widehat{\sigma}^2$}
   \nl Estimate $\bm{W}$ by solving generalized eigenvalue decomposition \eqref{eq:rkhs-eig} \;
   \nl Initialize $\bm{\Pi} \coloneqq \bm{\Gamma}_n \bm{K}\bm{W}$, $\bm{\Lambda} \coloneqq \bm{\Pi}^\top\bm{\Pi}$ \;
   \Repeat{log-likelihood \eqref{eq:log-lik} converges}{
   \nl $\widehat{\bm{V}}^{-1} \leftarrow \widehat{\sigma}^{-2}\left[\bm{I} - \bm{\Pi}\left(\widehat{\sigma}^2\widehat{\bm{\Sigma}}_{\beta}^{-1}+\bm{\Lambda}\right)^{-1}\bm{\Pi}^\top
   \right]$ \;
   \nl $\widehat{\bm{L}} \leftarrow \bm{I} - \bm{1}\bm{1}^\top\widehat{\bm{V}}^{-1} \left(\bm{1}^\top\widehat{\bm{V}}^{-1}\bm{1}\right)^{-1}$ \;
   \nl $\widehat{\bm{\alpha}} \leftarrow \left(\bm{\Pi}^\top\widehat{\bm{V}}^{-1}\widehat{\bm{L}}\bm{\Pi}+n\xi\bm{W}^\top\bm{K}\bm{W}\right)^{-1} \bm{\Pi}^\top\widehat{\bm{V}}^{-1}\widehat{\bm{L}}\bm{y}$ \;
   \nl $\widehat{c} \leftarrow \left(\bm{y}^\top - \widehat{\bm{\alpha}}^\top\bm{\Pi}^\top\right)
   \widehat{\bm{V}}^{-1}\bm{1} \left(\bm{1}^\top\widehat{\bm{V}}^{-1}\bm{1}\right)^{-1}$ \;
   \nl $\widehat{\bm{\Delta}} \leftarrow \left(\widehat{\bm{\Sigma}}_{\beta}^{-1} + \bm{\Lambda}/\widehat{\sigma}^2\right)^{-1}$ \;
   \nl $\widehat{\bm{\beta}} \leftarrow \widehat{\sigma}^{-2}\widehat{\bm{\Delta}}\bm{\Pi}^\top\left(\bm{y} - \bm{\Pi}\widehat{\bm{\alpha}} - \bm{1}\widehat{c}\right)$ \;
   \nl $\widehat{\bm{\Sigma}}_{\beta} \leftarrow \widehat{\bm{\beta}}\widehat{\bm{\beta}}^\top + \widehat{\bm{\Delta}}$ \;
   \nl $\widehat{\sigma}^2 \leftarrow \widehat{\sigma}^2 + n^{-1}\left(\left\|\bm{y} - \bm{\Pi}\left(\widehat{\bm{\beta}} + \widehat{\bm{\alpha}}\right) - \bm{1}\widehat{c}\right\|^2 - \widehat{\sigma}^4 \tr \widehat{\bm{V}}^{-1}\right)$ \;
   }
\caption{EM algorithm for learning the SIGP}
\label{alg:learn}
\end{algorithm}
\end{minipage}
\end{figure}

The M-step maximizes the expectation given in the E-step. The optimization is straightforward by setting the partial derivatives with respect to both $\bm{\Sigma}_\beta^{-1}$ and $\sigma^{-2}$ to zero. The resulting updates are given by
\begin{gather}
\label{eq:upd-sigma-beta}
\widehat{\bm{\Sigma}}_\beta \leftarrow
  \widehat{\bm{\beta}}\widehat{\bm{\beta}}^\top
+ \widehat{\bm{\Delta}}\\
\widehat{\sigma}^2 \leftarrow 
\widehat{\sigma}^2 + \frac{1}{n} 
\left(\left\|\widehat{\bm{r}} - \bm{\Pi}\widehat{\bm{\beta}}\right\|^2
- \widehat{\sigma}^4 \tr \widehat{\bm{V}}^{-1} \right). \label{eq:s2-est}
\end{gather}
Note that the random effect $\widehat{\bm{r}}$ also depends on the mean function which has parameters $\widehat{\bm{\alpha}}$ and $\widehat{c}$. These parameters can be optimized in the M-step as well. For example, the Tikhonov regularization \eqref{eq:reg-obj} with the a quadratic loss can be used:
\begin{align*}
\arg \min_{\bm{\alpha},c} \quad \frac{1}{n} \left(\bm{y} - \bm{\Pi}\bm{\alpha} - c\bm{1}\right)^\top\widehat{\bm{V}}^{-1} \left(\bm{y} - \bm{\Pi}\bm{\alpha} - c\bm{1}\right)
+ \xi \bm{\alpha}^\top\bm{W}^\top\bm{K}\bm{W}\bm{\alpha},
\end{align*}
where $\xi > 0$ is the regularization parameter. Denote by $\bm{L} \coloneqq \bm{I} - \bm{1}\bm{1}^\top\widehat{\bm{V}}^{-1}\left(\bm{1}^\top\widehat{\bm{V}}^{-1}\bm{1}\right)^{-1}$ the centering matrix, it is easy to obtain the following estimates
\begin{gather}
\label{eq:opt-mean}
\widehat{\bm{\alpha}} = \left(\bm{\Pi}^\top\widehat{\bm{V}}^{-1}\bm{L}\bm{\Pi} + n\xi\bm{W}^\top\bm{K}\bm{W}\right)^{-1}\bm{\Pi}^\top\widehat{\bm{V}}^{-1}\bm{L}\bm{y}\\
\widehat{c} = \bm{1}^\top\widehat{\bm{V}}^{-1}\left(\bm{y} - \bm{\Pi}\widehat{\bm{\alpha}}\right) \left(\bm{1}^\top\widehat{\bm{V}}^{-1}\bm{1}\right)^{-1}.
\end{gather}

\cref{alg:learn} gives the pseudo-code for the above EM algorithm. We remark that the computation of \cref{alg:learn} is efficient, requiring $O\left(n^2 m\right)$ time per iteration by taking advantage of the low-rank SDR parameterization. Furthermore, $\widehat{\bm{V}}^{-1}$ is computed via the Woodbury identity
\begin{align*}
\bm{V}^{-1} = \sigma^{-2} \left[\bm{I} - \bm{\Pi}\left(\sigma^2\bm{\Sigma}_\beta^{-1} + \bm{\Pi}^\top\bm{\Pi}\right)^{-1}\bm{\Pi}^\top\right],
\end{align*}
which reduces the computational complexity $O\left(n^3\right)$ of the inverse to $O\left(n^2 m\right)$. In practice, $m$ can be very small as we will show on several real datasets.

\section{Experiments}
\label{sec:exp}

We present experiments to 1) highlight the difference in the predictive distribution between the SIGP and the standard GP regression model; 2) illustrate the impact of the SDR approximation on the SIGP prediction; and 3) demonstrate that the SIGP with low-rank SDR approximation achieves competitive performance compared to state-of-the-art GP inference methods on a diverse collection of real-life datasets.

For the comparison, we consider several state-of-the-art GP inference methods, namely the Laplace’s approximation (Laplace) \citep{Rasmussen06}, Kullback-Leibler divergence minimization (KL) \citep{Nickisch08}, expectation propagation (EP) \citep{Minka01}, and fully independent training conditional (FITC) \citep{Snelson06}. The experiments are based on the GP implementation in GPML toolbox \citep{Rasmussen10} which is generally considered to be amongst the best implementation of these algorithms. The support vector machine (SVM) results are based on the {\tt fitcsvm} function from the Matlab. All methods use the radial basis kernel of which the parameters for SVM and SIGP are obtained via cross-validation. For all other GP inference methods, we optimize hyper-parameters using L-BFGS \citep{Liu89} for 1000 iterations. In addition to learning kernel hyper-parameters, we fit a linear mean function.

\subsection{Illustration on Synthetic Data}

Since the SDR approximation in the SIGP captures the functional dependence of $Y$ on $X$, the SIGP is particularly well-suited to modeling data with sampling biases.

\begin{figure}[h]
\centering
\includegraphics[scale=0.7]{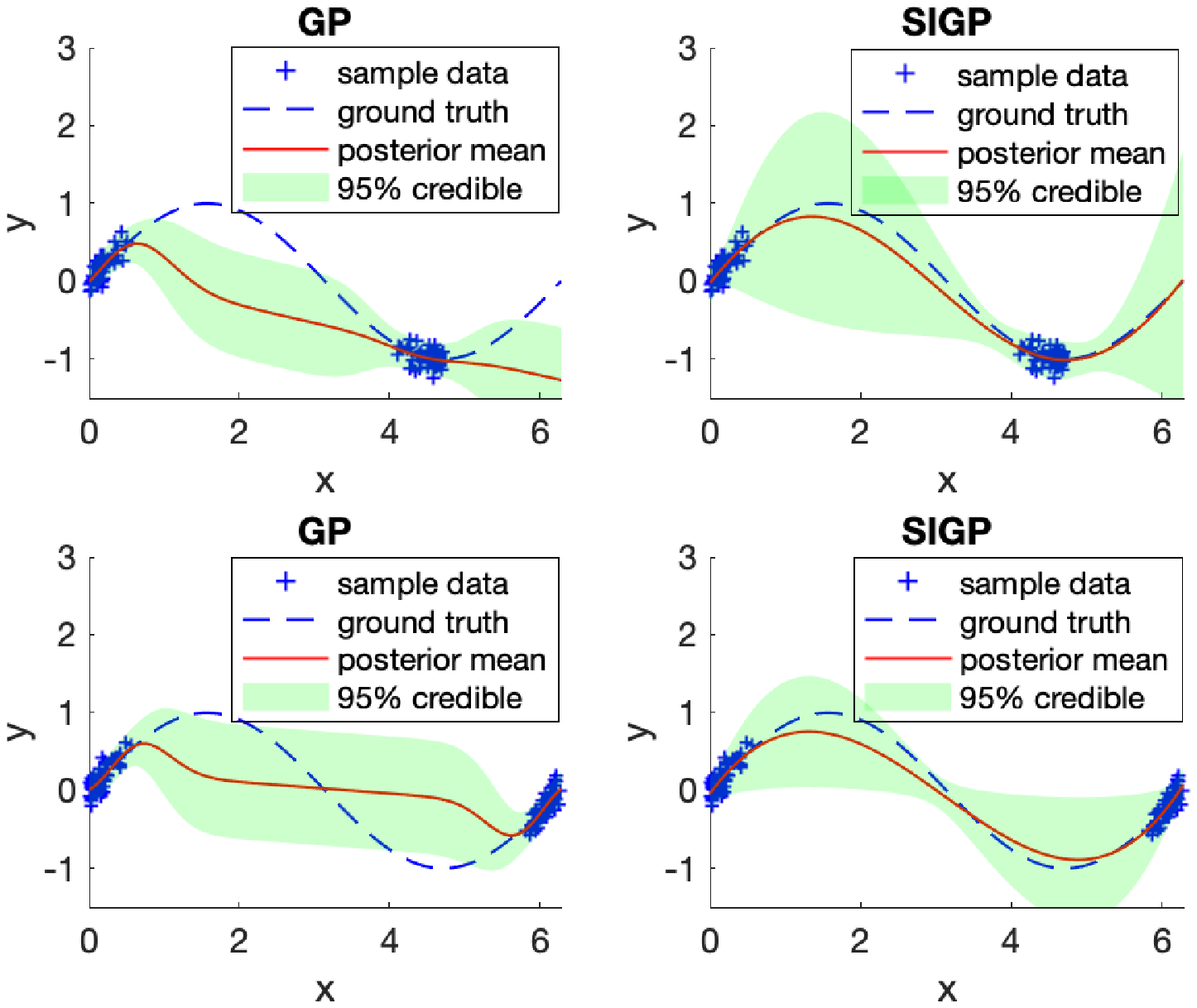}
\caption{Inference in a sinusoidal function with training data sampled at different locations.}
\label{fig:sinusoid}
\end{figure}

\cref{fig:sinusoid} depicts two samples at different locations, and the data is generated from a sine function with an additive Gaussian noise $\mathcal{N}\left(0,0.01\right)$. The posterior means and $95\%$ credible intervals produced by the GP and SIGP are shown. The SIGP recovered the ground truth sine function on both samples, and tend to give more realistic uncertainty estimates at the unseen positions. This suggests that the SIGP can be useful to address problems such as covariate shift.

\subsubsection{Impact of the SDR Approximation}
\label{sec:sdr-impact}

Now consider a classification example shown in \cref{fig:class-fit}. The toy dataset consists of  2D points from four classes. Note that each corner consists of points from two distinct classes. The figure shows the contours generated by the GP and SIGP. The SIGP-$m$ denote the SIGP with rank-$m$ SDR approximation. 

\begin{figure}[h]
\centering
\begin{adjustbox}{width=\columnwidth}
\input{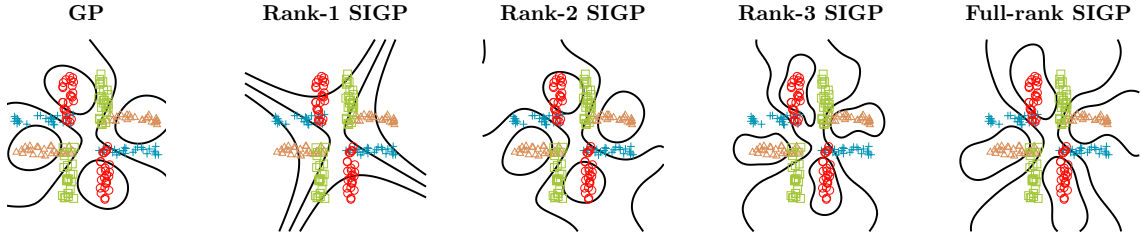}
\end{adjustbox}
\caption{Comparing the separating contours produced by the GP and SIGP on a 2D dataset with four classes marked by \textcolor[rgb]{0.85000,0.56000,0.35000}{$\vartriangle$}, \textcolor[rgb]{0.00000,0.58000,0.71000}{$+$}, \textcolor{lime!56!gray}{$\Box$}, and \textcolor{red!94!orange}{$\ocircle$}, respectively.}
\label{fig:class-fit}
\end{figure} 

For SIGP-$1$, the contours only separate the classes at the corners. This is because the basis functions of the SDR subspace correspond to principal components of the normal vectors (functions in the RKHS) to the contours \citep{Li11}. Consequently, a rank-$1$ SDR subspace is insufficient to separate the four classes. However, as we increase the SDR rank all classes are successfully separated by the contours of the SIGP. In general, a rank-$m$ SIGP suffices to classify at least $m+1$ classes.

\begin{figure}[h]
\centering
\input{sdr}
\caption{Projection of the data onto the first two dimensions of the SDR basis.}
\label{fig:sdr}
\end{figure}
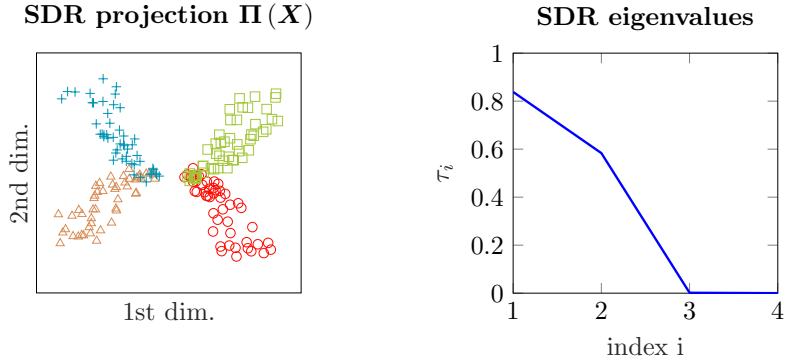

\cref{fig:sdr} plots the projection of the data onto the first two dimensions of the SDR subspace estimated by the SIGP. As can be seen, the data is well-separated in the SDR subspace representation of the SIGP. In addition, the eigenvalue drops quickly to zero at $\tau_3 = 0$, suggesting that the SIGP with rank-$2$ SDR approximation is sufficient for the data. Indeed, this is the case as shown in \cref{fig:class-fit}. 

\subsection{Results on Real-Life Data}

We compare the performance of the SIGP and several state-of-the-art methods for classification and regression on UCI datasets as well as real-world environmental datasets from WCCI-2006 Predictive Uncertainty in Environmental Modeling Competition. \cref{tbl:data-size} reports the total number of observations, the number of attributes, as well as the number of test cases for each dataset.

\begin{table}[htb]
\caption{The size as well as the training/testing splits of each dataset.}
\label{tbl:data-size}
\vskip 0.15in
\begin{center}
\begin{small}
\begin{sc}
\begin{tabular}{cccc|cccc}
\toprule
Dataset & \#Obser. & \#Attr. & \#Test & Dataset & \#Obser. & \#Attr. & \#Test\\
\midrule
Arcene  & 200      & 10000   & 100    & Cancer  & 699      & 10      & 200\\
Gisette & 7000     & 5000    & 1000   & Housing & 506      & 13      & 106\\
Madelon & 2600     & 500     & 600    & Temp    & 7117     & 106     & 3558\\
German  & 1000     & 24      & 300    & Wine    & 3098     & 11      & 1800\\
Heart   & 270      & 13      & 100    & Precip  & 7031     & 106     & 3515\\
\bottomrule
\end{tabular}
\end{sc}
\end{small}
\end{center}
\end{table}

\paragraph{Binary Classification} For the binary classification task, we use the F1 score, i.e., $2 \times \text{precision}\times\text{recall} / \left(\text{precision} + \text{recall}\right)$, as the balanced accuracy metric. Since the classification is binary, we use the SIGP with rank-$1$ SDR approximation.

\begin{table*}[htb]
\caption{Comparing F1 scores in binary classification.}
\label{tbl:F1}
\vskip 0.15in
\begin{center}
\begin{small}
\begin{threeparttable}
\begin{sc}
\begin{tabular}{lcccccc}
\toprule
Dataset & Laplace & KL & EP & FITC\tnote{1} & SVM & SIGP-1\\
\midrule
Arcene  & 0.8235 & 0.8269 & 0.8235 & 0.8235 & 0.8352 & {\bf 0.8571}\\
Gisette & 0.9570 & 0.9541 & 0.9571 & 0.9571 & 0.9670 & {\bf 0.9780}\\
Madelon & 0.5695 & 0.5695 & 0.5695 & 0.5695 & 0.5990 & {\bf 0.6367}\\
German & 0.6211 & 0.6211 & 0.6211 & 0.6125 & 0.6182  & {\bf 0.6424}\\
Heart  & 0.8409 & 0.8409 & 0.8409 & 0.8506 & {\bf 0.8605} & {\bf 0.8605}\\
Cancer & 0.9425 & 0.9213 & 0.9438 & 0.9778 & 0.9778 & \bf 0.9888\\
\bottomrule
\end{tabular}
\end{sc}
\begin{tablenotes}
\item[1] Using $\floor*{n/2}$ pseudo-inputs.
\end{tablenotes}
\end{threeparttable}
\end{small}
\end{center}
\end{table*}

\cref{tbl:F1} reports the F1 scores on the benchmark datasets. From the table, the SIGP performs competitively compared to state-of-the-art methods, particularly on the high-dimensional dataset {\tt Arcene}.

\paragraph{Regression}
For regression, we report the negative log predictive density (NLPD) and mean squared error (MSE) for the SIGP and the GP with the FITC inference which is known to yield better NLPD than the other inference methods \citep{Snelson06a}.

\begin{table*}[h]
\caption{Prediction performance on the held-out validation data. FITC-$t$ denotes the FITC method using $t$ pseudo-inputs. Since the {\tt Housing} dataset is relatively small, the entire training data is used.}
\label{tbl:reg}
\begin{center}
\begin{small}
\begin{sc}
\begin{tabular}{lcccccccccc}
\toprule
\multirow{2}{*}{Method} & 
\multicolumn{2}{c}{Housing}
& \multicolumn{2}{c}{Temp}
& \multicolumn{2}{c}{Wine} 
& \multicolumn{2}{c}{Precip}\\
\cline{2-3} \cline{4-5} \cline{6-7} \cline{8-9}
 & NLPD & MSE & NLPD & MSE & NLPD & MSE & NLPD & MSE\\
\midrule
Linear Reg.\ & 3.2363 & 37.8938 & 0.1265 & 0.0754 & 1.1052 & 0.5340 & 1.8268 & 2.2607\\
FITC-$1500$  & 3.1200 & 28.8048 & 0.0522 & 0.0649 & 1.1002 & 0.5417 & 1.7258 & 1.8298\\
FITC-$2000$  & 3.1200 & 28.8048 & 0.0527 & 0.0650 & 1.0972 & 0.5724 & 1.7229 & 1.8353\\
FITC-$2500$  & 3.1200 & 28.8048 & 0.0520 & 0.0647 & 1.0966 & 0.5628 & 1.7204 & 1.8302\\
SIGP-$1$     & 2.7756 & 15.0003 & 0.0531 & 0.0640 & 1.0953 & 0.5228 & 1.7221 & 1.8349\\
SIGP-$2$     & \bf 2.7459 & \bf 14.2078 & 0.0513 & 0.0628 & \bf 1.0905 & \bf 0.5177 & \bf 1.7135 & \bf 1.7664\\
SIGP-$3$     & 2.8393 & 16.5767 & \bf 0.0498 & \bf 0.0617 & 1.0911 & 0.5181 & 1.7163 & 1.7783\\
\bottomrule
\end{tabular}
\end{sc}
\end{small}
\end{center}
\end{table*}

\cref{tbl:reg} compares the NLPD (smaller is better) as well as MSE on the held-out validation data. The experiment shows that higher rank SDR approximation in the SIGP may not necessarily improve the predictive performance. This is reasonable as increasing the rank may also overfit the data, and the optimal rank actually depends on the structure of data, e.g., the number of classes as discussed in \cref{sec:sdr-impact}. The result in \cref{tbl:reg} suggests that the SIGP performs consistently in both MSE as well as NLPD, and is a state-of-the-art method for regression.

\section{Conclusions}
\label{sec:concl}

In this paper, we introduce novel non-parametric stochastic regression
models based on integral representations of Gaussian processes. We
provide a characterization of the sample paths of these GP models with
respect to the RKHS that contains the sample paths. The theoretical
ideas developed in formulating the novel GP is of interest in
itself. We then show how we can use the GP defined by the integral
representations for computationally efficient and statistically
accurate non-parametric regression using a data-dependent kernel
model. We illustrate the practical utility of this via results on
simulated and real data. From a machine learning perspective we
provide a way to efficiently infer hyper-parameters in a
data-dependent way that takes prediction into account for Gaussian
processes.

We suspect that there are extensions to the IGP and our sample-based
implementation of the IGP both from a theoretical and practical
perspective. In addition, considering powers of integral operators
fits into the perspective of understanding the power of deep learning
as interpolation \citep{Belkin18} and the
idea of analyzing deep learning via kernel models that interpolate
\citep{Belkin18a}.

\section*{Code and Data}
The datasets as well as the sample-based implementation of the IGP are available on the Git repository: \url{https://github.com/ZilongTan/sigp}.

\section*{Acknowledgements}

Z.T.\ was supported in part by grant NSF DMS-1713012. S.M.\ would like to acknowledge the support of grants NSF IIS-1546331, NSF DMS-1418261, NSF IIS-1320357, NSF DMS-1045153, and NSF DMS-1613261.

\bibliographystyle{apalike}
\bibliography{ref}

\end{document}

%% file: def.tex
\usepackage[utf8]{inputenc}
\usepackage{natbib}

\usepackage{amsmath}
\usepackage{amsfonts}
\usepackage{amssymb}
\usepackage{amsthm}
\usepackage{bm}
\usepackage{mathrsfs}
\usepackage{mathtools}
\usepackage{wasysym}

\usepackage{fancyhdr}
\usepackage{breakcites}
\usepackage{microtype}
\usepackage{graphicx}
\usepackage{subfigure}
\usepackage{booktabs}
\usepackage{multirow}
\usepackage{threeparttable}

\usepackage[ruled]{algorithm2e}

\usepackage{pgfplots}
\usepackage{tikzscale}
\pgfplotsset{compat=newest}
\usetikzlibrary{plotmarks}
\usetikzlibrary{arrows.meta}
\usepgfplotslibrary{patchplots}
\usepackage{grffile}

\usepackage[breaklinks,hidelinks,bookmarks=true]{hyperref}
\usepackage[capitalize,nameinlink,noabbrev]{cleveref}

\usepackage{adjustbox}

\def\ci{\perp\!\!\!\perp}
\newtheorem{theorem}{Theorem}
\newtheorem{lemma}{Lemma}
\newtheorem{prop}{Proposition}
\newtheorem{cond}{Condition}

\DeclareMathOperator{\tr}{tr}

\DeclarePairedDelimiter\floor{\lfloor}{\rfloor}

\newcommand{\eat}[1]{}

\makeatletter
\renewcommand*\env@matrix[1][*\c@MaxMatrixCols c]{
  \hskip -\arraycolsep
  \let\@ifnextchar\new@ifnextchar
  \array{#1}}
\makeatother

%% file: sdr.tex
\definecolor{mycolor1}{rgb}{0.85000,0.56000,0.35000}
\definecolor{mycolor2}{rgb}{0.00000,0.58000,0.71000}
\colorlet{mycolor3}{lime!56!gray}
\colorlet{mycolor4}{blue}

\begin{tikzpicture}[scale=0.9]

\begin{axis}[%
width=1.539in,
height=1.397in,
at={(0in,0.424in)},
scale only axis,
xmin=-0.0004,
xmax=0.0004,
xtick={\empty},
xlabel style={font=\color{white!15!black}},
xlabel={1st dim.},
ymin=-0.0004,
ymax=0.0004,
ytick={\empty},
ylabel style={font=\color{white!15!black}},
ylabel={2nd dim.},
axis background/.style={fill=white},
title style={font=\bfseries},
title={SDR projection ${\bm{\Pi}}\left(\bm{X}\right)$}
]
\addplot [color=mycolor1, draw=none, mark=triangle, mark options={solid, mycolor1}, forget plot]
  table[row sep=crcr]{%
-0.000333855215248727	-0.000138261158008386\\
-0.0001752776218707	-1.27067250838249e-05\\
-0.000163066742273192	-5.36066669230926e-05\\
-0.000320310717311221	-0.000198856758312621\\
-0.000228961217826742	-0.000121996191317328\\
-0.000218259813284751	-0.000226313300689959\\
-0.000220656030982509	-4.35287370402395e-05\\
-6.81015201316361e-05	-2.31361086782121e-06\\
-0.000199174069243566	-6.08573739497838e-05\\
-0.000292551239136385	-0.000143057352024963\\
-0.000137138238864064	-1.09318488271773e-05\\
-0.000289570353385319	-0.00021110350698065\\
-0.000101136263720068	-2.23296287210507e-05\\
-0.000223741085995342	-7.60790576011e-05\\
-0.000257650098553111	-0.000208662030642179\\
-0.000120258674554272	1.25624524433911e-05\\
-0.000209198959348869	-3.36829453668161e-05\\
-0.00023691153195851	-0.000201806308765566\\
-9.09423122800655e-05	-1.28217796984489e-05\\
-0.000221493990039223	-0.000121824874707465\\
-0.000169780170639449	-3.94072855428195e-05\\
-0.000262526673952638	-0.000137509418220544\\
-0.000273599670669269	-0.000212200609449643\\
-0.000240947936737699	-0.00018135744400786\\
-4.91695637052544e-05	-1.02865160808282e-06\\
-0.000159748247001764	-0.000118950721814821\\
-0.000125837595274032	-4.53971588547845e-05\\
-0.000153878500193848	-8.82131572676654e-05\\
-0.000158919389582516	-4.72612855064147e-05\\
-0.000234272141524814	-0.000151380561809414\\
-0.000158372914736586	-0.000125431153351997\\
-0.000216122332790024	-6.38003135373433e-05\\
-0.000226237438266374	-8.42163705567771e-05\\
-8.34160446027322e-05	-1.92910071678416e-05\\
-4.2400722190626e-05	-1.97718079100448e-05\\
-0.000173122829330272	-0.000166019681938271\\
-0.000119326516542349	-2.00105359360656e-05\\
-0.000312216995249037	-0.000188108013881161\\
-9.68271949393957e-05	-1.6033207834485e-05\\
-0.00020638172438318	-0.000191746397038795\\
-9.84846515197337e-05	-6.71300954857383e-05\\
-0.000272227517816465	-0.000123857893945918\\
-0.000100375054218716	-6.95666043568019e-05\\
-0.000132303438655813	-3.22065811579641e-05\\
-8.75842738227579e-05	-6.52996852994532e-06\\
-0.00032844019014075	-0.00023470796933455\\
-4.56907904833971e-05	-4.95260873404536e-06\\
-0.000263171800808402	-0.000206951542983011\\
-5.93837825004643e-05	-2.27297484575321e-06\\
-0.000234282098025364	-0.000168138266426232\\
};
\addplot [color=red!94!orange, draw=none, mark=o, mark options={solid, red!94!orange}, forget plot]
  table[row sep=crcr]{%
0.000235260895500423	-0.000230904480357044\\
0.000185348920933877	-0.000106018897918059\\
0.0001563086587488	-0.000245367322262821\\
0.000178002898282887	-0.00016247733695573\\
0.000135289901634108	-5.04150261641374e-05\\
0.000123886417997691	-8.03267240151787e-05\\
0.000113403512011717	-7.62553865743958e-05\\
8.27256796934201e-05	-2.73332621859278e-05\\
0.000170511208329866	-6.98808755828577e-05\\
0.000186335586648307	-0.000200530475099343\\
6.75329483260814e-05	-1.20307437696023e-05\\
0.00014768557750816	-0.00015585319830435\\
0.000149308579454803	-3.64138571446285e-05\\
0.000101271981273948	-3.99182628678683e-05\\
0.000163288977134225	-0.000121372199225325\\
0.000308371567746884	-0.000255714684865412\\
0.000269124554137676	-0.000237812493980393\\
0.000115260104153945	-2.92996418267477e-05\\
0.000252674926942935	-0.000180633054398185\\
7.59274496524758e-05	-7.19250453829273e-06\\
0.000134999746051622	-0.000134536796177364\\
5.56860306326369e-05	-5.66806047790896e-06\\
0.000136118231979574	-6.80005810313585e-05\\
6.66691638595467e-05	-1.5979738944056e-05\\
0.000300340693621726	-0.000233265699860919\\
0.000141078270206007	-5.71929236148837e-05\\
0.000204572842071623	-0.000278537371527167\\
0.000287093162048645	-0.000267303171116561\\
0.000133741162942254	-0.000195169169480729\\
0.000156335897144023	-8.23544687109422e-05\\
0.000238920412878123	-0.000251222114531205\\
7.35494483039819e-05	-2.92489571063331e-05\\
0.000249263382528041	-0.000261209813250905\\
0.00011802965907235	-7.69173858873817e-05\\
0.000137304901899434	-2.87601162007862e-05\\
0.000159801875585432	-0.000260684842827067\\
9.09265640193945e-05	-1.15654599048396e-05\\
0.000219161201322075	-0.000116562410439437\\
0.000209903543684417	-0.000101065363339264\\
4.74176315322301e-05	-1.41827648720593e-05\\
8.04222073914482e-05	-6.9160482779976e-06\\
7.84810929334923e-05	-6.0739454683046e-05\\
8.45092811842168e-05	-6.32228936082272e-05\\
7.32212301716743e-05	1.54632832369634e-05\\
0.000201614055978918	-0.000249591128612995\\
9.88505192798939e-05	-4.74959987342052e-06\\
0.000182703590797789	-0.000242067124834394\\
0.000122110200269076	-5.29809614297174e-05\\
0.000154160044731775	-5.97282600639947e-05\\
0.000108088547807337	-6.40992049516234e-05\\
};
\addplot [color=mycolor2, draw=none, mark=+, mark options={solid, mycolor2}, forget plot]
  table[row sep=crcr]{%
-0.000156091077523082	7.67208662102622e-05\\
-0.000140822072373539	4.67766406200967e-05\\
-0.000166926153291633	0.000227315982764641\\
-5.65789435125925e-05	-5.16101887181652e-06\\
-0.000198550130729074	0.000313727556829029\\
-0.000147009836271179	9.66383596961429e-05\\
-0.000219321838177377	0.00016341854061114\\
-4.60058316842641e-05	1.08094866687834e-05\\
-0.000124665891458898	7.96270418521271e-05\\
-4.83030032207244e-05	1.53328182674466e-05\\
-0.000147071937787634	0.000217353524485873\\
-0.000136862265023213	9.55221392094301e-05\\
-0.000161902896599089	0.000262870500852262\\
-0.000128338697058991	5.53107951344566e-05\\
-0.000207322671156803	0.000128807208744026\\
-0.000178205388879974	0.000165658531025142\\
-0.000136329370747697	0.000123085141169284\\
-9.59354441162461e-05	5.80025106989288e-05\\
-0.000199049225734468	0.000238170843597013\\
-0.000103190873954482	0.000113299773415257\\
-3.72045708499828e-05	-4.95464641655454e-06\\
-0.000111839937410465	3.1154695787414e-05\\
-0.000104525827408742	2.71674946492882e-05\\
-3.91840940546374e-05	-1.19028700214808e-06\\
-0.000160293590153641	6.64292170444552e-05\\
-8.7125829859985e-05	3.86707715092655e-05\\
-6.81645794714496e-05	-2.98986274322284e-05\\
-0.000110573174263231	9.36548973201801e-05\\
-0.000201654619571806	0.000129582686093841\\
-0.000305914190690076	0.000271203817051306\\
-0.000231057244956486	0.00023434487492862\\
-2.88056127584897e-05	-1.05882565093258e-05\\
-0.000149252829332331	0.000142719094377305\\
-0.000184582181249249	0.000198345996444726\\
-0.000196643148766793	0.000120929852128852\\
-8.58697868509493e-05	-1.50068127256314e-05\\
-0.000237947340811645	0.000265068058491745\\
-0.000128512383475496	3.89629047622905e-05\\
-0.0001680038393944	0.000118507556419808\\
-0.0001049929034309	-6.4092485005738e-06\\
-0.000134729289169712	9.06817615591173e-05\\
-5.26422869190523e-05	-2.84622380208679e-06\\
-0.000283390313536835	0.000269848562691364\\
-0.000184155753775488	0.000116135860518555\\
-0.000158542282150862	0.000287148908922074\\
-0.000196067940830465	0.000127066785061726\\
-0.000187638759751262	0.000127401776133353\\
-0.00022766028436346	0.000233184676156967\\
-0.000167600705655097	4.24481557498565e-05\\
-0.000323867042285852	0.000246887167659407\\
};
\addplot [color=mycolor3, draw=none, mark=square, mark options={solid, mycolor3}, forget plot]
  table[row sep=crcr]{%
0.000249705963015488	0.000127971276245516\\
0.000184914155043459	0.000115710561327976\\
0.000256589942210945	0.000123901260861897\\
0.000163388070463524	4.07834754355504e-05\\
8.41231369956745e-05	-8.81227449035533e-06\\
0.000181541211778184	5.38203913403876e-05\\
0.0003295423164541	0.000172885397487202\\
0.000143673984446508	6.41902437021442e-05\\
0.000246658665049161	0.000112313148644893\\
0.000241882009583602	7.54102947384061e-05\\
0.000102867741856688	2.22817836223361e-05\\
0.000177092722760334	1.59326589976394e-05\\
0.000221964156577108	5.04041613949373e-05\\
8.17293982526952e-05	-1.85695125634145e-05\\
0.000205866579201325	0.000176746675635841\\
8.06511913888597e-05	-1.56283790406615e-05\\
0.000261832760969428	0.000162471544810534\\
0.000212393988477862	7.2485309172154e-05\\
6.31315486581879e-05	-2.65208611991835e-05\\
0.000322559663310924	0.000251560742539209\\
0.000239209296657162	0.000217426861536889\\
7.71780199096074e-05	-7.46865192849065e-06\\
0.000324114395505229	0.000264051063058037\\
0.000207027022453187	9.4996185760075e-05\\
0.000103947337114697	2.35954385910827e-05\\
0.000136794521436128	1.23894211022771e-05\\
0.000105319357853309	4.94925432430758e-07\\
5.66015128853082e-05	-1.21749268688907e-05\\
0.00018425582034189	0.000212216862802423\\
0.000157366718695903	0.000100336525095968\\
0.000196085245143591	4.24184485836659e-05\\
0.00031179905685729	0.000179162721670176\\
6.34117129575679e-05	-1.97915014036343e-06\\
0.000213820365355816	0.000248390910149428\\
0.000127238606187938	1.53323304569444e-05\\
0.000186286259845583	0.000176392467482969\\
7.72704427735799e-05	-1.73559881506513e-05\\
0.000149024788631018	1.61328004004995e-05\\
0.000141946280951516	9.78712417781095e-05\\
0.000295743520848388	0.000212702492316686\\
0.000152744543826586	0.000127231604869539\\
0.00022333577810677	0.000218720841366615\\
0.0001531538537596	5.70201627112595e-05\\
0.000117459850027464	7.80510708213477e-05\\
0.000168254492407914	0.000104200173045638\\
0.000294588130466442	0.000145347256635462\\
0.000260624665379491	0.000254360035439872\\
0.000285066946749048	0.000137008520614721\\
7.23621235900973e-05	-1.38227542868721e-05\\
0.000193494977580778	4.31102218485559e-05\\
};
\end{axis}

\begin{axis}[%
width=1.539in,
height=1.397in,
at={(2.772in,0.424in)},
scale only axis,
xmin=1,
xmax=4,
xlabel style={font=\color{white!15!black}},
xlabel={index i},
ymin=-1.96618071979238e-06,
ymax=1,
ylabel style={font=\color{white!15!black}},
ylabel={$\tau_i$},
axis background/.style={fill=white},
title style={font=\bfseries},
title={SDR eigenvalues}
]
\addplot [color=mycolor4, line width=1.0pt, forget plot]
  table[row sep=crcr]{%
1	0.83825873598142\\
2	0.583554159979614\\
3	0.00167677441776337\\
4	-1.96618071979238e-06\\
};
\end{axis}

\begin{axis}[%
width=4.764in,
height=2.042in,
at={(0in,0in)},
scale only axis,
xmin=0,
xmax=1,
ymin=0,
ymax=1,
axis line style={draw=none},
ticks=none,
axis x line*=bottom,
axis y line*=left
]
\end{axis}
\end{tikzpicture}%